\PassOptionsToPackage{a4paper,margin=2.5cm}{geometry}
\documentclass[final,1p,times,authoryear]{elsarticle}

\usepackage{amsmath,amsfonts,bm, amssymb}
\usepackage{cite}
\usepackage{algorithmic}
\usepackage{float}
\usepackage{graphicx}
\usepackage{textcomp}
\usepackage{xcolor}
\usepackage{natbib}   
\usepackage{todonotes}
\usepackage{booktabs} 
\usepackage{tikz}
\usepackage{algorithm}
\usepackage{mdframed} 
\usepackage{subcaption}
\usepackage{caption}
\usepackage{lipsum} 
\usetikzlibrary{shapes, arrows.meta, positioning, backgrounds, fit}

\def\eqref#1{equation~\ref{#1}}

\def\1{\bm{1}}

\DeclareMathAlphabet{\mathsfit}{\encodingdefault}{\sfdefault}{m}{sl}
\SetMathAlphabet{\mathsfit}{bold}{\encodingdefault}{\sfdefault}{bx}{n}

\def\gA{{\mathcal{A}}}

\def\gC{{\mathcal{C}}}

\def\gF{{\mathcal{F}}}

\def\gM{{\mathcal{M}}}
\def\gN{{\mathcal{N}}}

\def\gP{{\mathcal{P}}}

\def\gR{{\mathcal{R}}}
\def\gS{{\mathcal{S}}}
\def\gT{{\mathcal{T}}}

\DeclareMathOperator*{\argmax}{arg\,max}

\usepackage[most]{tcolorbox}
\newtcolorbox{nbox}[1][]{
  enhanced,
  fonttitle=\scshape,
  #1
}

\usepackage{amsthm}

\newtheorem{theorem}{Theorem}
\newtheorem{lemma}{Lemma}

\newtheorem{remark}{Remark}

\newtheorem{assumption}{Assumption}

\newcommand{\dirac}{\mathrm{Dirac}}
\newcommand{\diracP}[1]{\dirac\!\left( #1 \right)}

\usepackage[colorlinks=true,      
            linkcolor=blue,       
            citecolor=blue,       
            urlcolor=blue]{hyperref}

\journal{Performance Evaluation}

\begin{document}

\begin{frontmatter}

\title{Online Matching via Reinforcement Learning:\\ An Expert Policy Orchestration Strategy} 

\author[1]{Chiara Mignacco}
\ead{chiara.mignacco@universite-paris-saclay.fr}

\author[2]{Matthieu Jonckheere}
\ead{matthieu.jonckheere@laas.fr}

\author[1]{Gilles Stoltz}
\ead{gilles.stoltz@universite-paris-saclay.fr}

\affiliation[1]{
    organization={Université Paris-Saclay, CNRS, Inria, Laboratoire de mathématiques d’Orsay},
    addressline={},
    city={Orsay},
    postcode={91405},
    country={France}
}

\affiliation[2]{
    organization={LAAS, Université de Toulouse, CNRS},
    addressline={},
    city={Toulouse},
    country={France}
}

\begin{abstract}

Online matching problems arise in many complex systems, from cloud services and online marketplaces to organ exchange networks, where timely, principled decisions are critical for maintaining high system performance. Traditional heuristics in these settings are simple and interpretable but typically tailored to specific operating regimes, which can lead to inefficiencies when conditions change. We propose a reinforcement learning (RL) approach that learns to orchestrate a set of such expert policies, leveraging their complementary strengths in a data-driven, adaptive manner. Building on the Adv2 framework \citep{Jonckheere2024}, our method combines expert decisions through advantage-based weight updates and extends naturally to settings where only estimated value functions are available. We establish both expectation and high-probability regret guarantees and derive a novel finite-time bias bound for temporal-difference learning, enabling reliable advantage estimation even under constant step size and non-stationary dynamics. To support scalability, we introduce a neural actor-critic architecture that generalizes across large state spaces while preserving interpretability. Simulations on stochastic matching models, including an organ exchange scenario, show that the orchestrated policy converges faster and yields higher system level efficiency than both individual experts and conventional RL baselines. Our results highlight how structured, adaptive learning can improve the modeling and management of complex resource allocation and decision-making processes.

\end{abstract}



\end{frontmatter}

\section{Introduction}
Matching problems lie at the heart of many combinatorial optimization challenges in computer science, from network design to resource allocation. A matching in a graph is a subset of edges such that no two edges share a vertex. The task of finding matchings of maximum cardinality or weight has driven decades of algorithmic research.

More recently, online matching problem, where decisions must be made sequentially as information arrives, have gained significant prominence due to applications in online marketplaces, supply chain logistics, and most critically, organ exchange programs. In these domains, efficient real-time decision-making is essential, yet complicated by the stochastic and dynamic nature of inputs.

In their seminal paper, \citet{karp1990optimal} introduced an elegant algorithm for unweighted bipartite matching achieving an optimal competitive ratio.  This opened a rich line of research: \citet{feldman2009online} and \citet{mahdian2011online} explored stochastic arrivals and probabilistic edge existence, while recent work such as \citet{brubach2021improved} expanded the theory to weighted and vertex-weighted settings. Some model-free approaches have been considered in e.g., \citet{zhang2024online} and \citet{jordan}. On the more mathematical side of the spectrum, structural and long-term stability properties have also been explored, particularly in infinite-state and random graph models \citep{mairesse2016stability,Comte17,comte2021stochastic,noiry2021online,soprano2023online,cherifa2025online}.

Among the most impactful applications of online matching is organ exchange, where matching donors and recipients is a life-saving endeavor. Here, decisions must not only be effective, but transparent and interpretable, grounded in ethical and clinical criteria. Traditional matching algorithms, often based on static or greedy heuristics, lack the adaptability required to optimize outcomes in such dynamic environments. Moreover, rigid adherence to fixed policies may result in longer waiting times and fewer successful transplants.

To bridge the gap between efficiency and interpretability, we explore a Reinforcement Learning (RL)-based orchestration framework that builds on fixed, human-understandable policies and learns how to combine them adaptively. Our method strikes a balance between transparency, by relying on interpretable building blocks, and performance, by optimizing over time which policy to trust under which conditions.

This idea aligns with emerging work on RL policy orchestration (also known as policy aggregation or imitation learning), where domain-specific heuristics are treated as expert policies and learned combinations are used to navigate complex environments \citep{CBL06,CKA20,Liu23}. Notably, our work builds on the Adv2 framework of \citet{Jonckheere2024}, which casts policy selection as an adversarial aggregation problem using $Q$-values or advantage functions, and provides regret-based performance guarantees.

\subsection*{Contributions}

We advance the orchestration of interpretable policies by providing both strong theoretical underpinnings and a scalable, practical framework:

\paragraph{Theoretical Guarantees for Policy Orchestration with Learned Advantages}

We extend the Adv2 framework to settings where perfect value estimates are unavailable and expert policies operate in high-dimensional or dynamic environments. Our orchestration strategy combines multiple interpretable experts into a single decision-making agent using robust weight updates based on estimated advantage functions.

We provide rigorous theoretical guarantees for this approach:
\begin{itemize}
\item A control-in-expectation theorem (Theorem~\ref{th:cesaro}) showing that the average performance of the mixture policy converges to that of the best convex combination of experts, up to an $\mathcal{O}(1/T)$ regret term.
\item A high-probability variant (Theorem~\ref{th:HP-estim}) that offers similar guarantees with probability $1 - \delta$, better suited for real-world applications.
\item A finite-time bias bound for temporal difference learning (Lemma~\ref{lemma:adv-bias-bound}), which characterizes how the TD bias in estimated advantages contracts geometrically over time. This result holds even under non-stationary sampling and constant step sizes, and supports dynamic policy learning by enabling reliable use of estimated advantages in our orchestration strategy.
\end{itemize}

These results build upon and significantly generalize the idealized assumptions in \citet{Jonckheere2024}, providing a foundation for learning-based orchestration under realistic conditions.

\paragraph{A Scalable and Interpretable Framework for Real-Time Policy Learning}

We implement our theoretical insights in a practical, scalable orchestration framework tailored to high-dimensional and continuous environments. Specifically, we propose a new neural-based actor critic architecture:
\begin{itemize}
\item The critic estimates the advantage of each expert policy in the current state.
\item The actor produces a probability distribution over experts, forming a learned mixture policy.
\end{itemize}
This architecture enables generalization across states and supports real-time adaptation, while maintaining interpretability by operating over expert-defined primitives. We use feedforward networks for tractability but the framework is extensible to domain-specific inductive biases (e.g., attention or graph structures), particularly in structured applications such as organ exchange.

Altogether, our work provides a theoretically grounded and practically scalable solution for orchestrating interpretable policies in complex, dynamic environments. By combining RL-based learning with structural insights and finite-time guarantees, we offer a promising direction for real-time decision-making in high-stakes domains like organ exchange, where adaptivity, transparency, and performance must go hand in hand.

\paragraph{Organization of the paper} Section \ref{sec:orchestration} formalizes the expert orchestration framework and defines the performance objectives. Section \ref{sec:learning-strat} introduces the learning strategies and provides convergence guarantees. Section \ref{sec:implementation_schemes} describes the practical implementation schemes, including tabular and neural network approaches. Section \ref{sec:stoch_model} presents the stochastic matching model we apply our methods to, and Section \ref{sec:simulations} reports simulation results demonstrating performance improvements. We conclude with a discussion of future directions.

\section{Orchestration of expert policies: setting and objectives}\label{sec:orchestration}

We revisit the general framework for expert orchestration, following the setup of \citet{jonckheere2023symphonyexpertsorchestrationadversarial}. We consider a Markov decision process (MDP) with finite state space $\gS$, action space $\gA$, transition kernel $\gT(s'|s,a)$, and bounded reward function $\gR(s,a) \in [0,1]$. For a policy $\pi$, we define the (discounted) value functions
\[
V^\pi(s_0) = \mathbb{E}^{(s_0, \pi)} \left[ \sum_{t=0}^\infty \gamma^t r_t  \right], \quad
Q^\pi(s_0,a) = \mathbb{E}^{(s_0, \pi)} \left[ \sum_{t=0}^\infty \gamma^t r_t \,\big|\, a_0 = a \right],
\]

with advantage $A^\pi(s,a) := Q^\pi(s,a) - V^\pi(s)$ and discount factor $\gamma \in (0,1)$. By the law of total expectation, we have $\sum_{a} \pi(a \mid s) A^\pi(s,a) = 0$.

We now fix a set $\Pi = \{\pi_1, \dots, \pi_K\}$ of $K \geq 2$ stationary expert policies. A state-dependent distribution $q(\cdot \mid s)$ over indices $[K]$ induces the mixture policy $q\Pi(\cdot \mid s) = \sum_{k} q(k \mid s) \pi_k(\cdot \mid s)$; sampling from $q\Pi$ amounts to first selecting $k \sim q(\cdot \mid s)$, then $a \sim \pi_k(\cdot \mid s)$. The class of all such mixtures is
\[
\gC(\Pi) := \{ q\Pi \mid q(\cdot \mid s) \in \gP([K]) \text{ for all } s \in \gS \}.
\]

Our goal is to find $q^\star$ maximizing performance within this class, i.e., $V_{q^\star\Pi}(s) = \max_q V_{q\Pi}(s)$ for all $s$. While standard RL aims for low regret relative to the global optimum $\pi^\star$, here we focus on minimizing regret relative to the best-in-class mixture policy $q^\star\Pi$, defined as the cumulative performance gap:
\[
\sum_{t=1}^T \left( V_{q^\star\Pi}(s_1) - V_{q_t\Pi}(s_1)\right),
\]
where $q_t$ is the policy at time $t$. This notion aligns with regret in online learning and underpins our theoretical results (Theorems~\ref{th:cesaro} and \ref{th:HP-estim}).

\begin{tcolorbox}[title=Aggregation of Expert Policies] \label{box}
\textbf{MDP}: state space \( \gS \), action space \( \gA \), transition kernel \( \gT: \gS \times \gA \to \gP(S) \), reward function \( \gR: \gS \times \gA \to \gP([0, 1]) \), where \( r: \gS \times \gA \to [0, 1] \) is the mean-payoff function. We have \( K \) expert policies \( \pi_1, \dots, \pi_K: \gS \to \gP(\gA) \).

\textbf{Initialization:} Start from an initial state \( s_0 \) and initial weights \( q_0 \in \gP([K])^S \).

For each round \( t = 0, 1, 2, \dots \), the process follows these steps:
\begin{enumerate}
    \item The learner:
    \begin{enumerate}
        \item Observes the current state \( s_t \),
        \item Chooses a policy index \( k_t \sim q_t(\cdot | s_t) \),
        \item Picks an action \( a_t \sim \pi_{k_t}(\cdot | s_t) \).
    \end{enumerate}
    \item The learner receives a reward \( r_t \sim \gR(s_t, a_t) \), with expected reward \( r(s_t, a_t) \).
    \item The state updates according to \( s_{t+1} \sim \gT(\cdot | s_t, a_t) \).
    \item The learner updates the state-dependent weights \( q_{t+1} \in \gP([K])^S \).
\end{enumerate}

\textbf{Objective:} 

 Maximize \( V_{q_T \Pi}(s) \) for all \( s \in \gS \) by adaptively combining experts, minimizing cumulative regret $\sum_{t=1}^T \left(V_{q^\star\Pi}(s) - V_{q_t\Pi}(s)\right)$ over time.

\end{tcolorbox}

We summarize our setting and goal in the box above.

\section{Learning Strategies and Convergence Rates} \label{sec:learning-strat}

This section introduces learning strategies for expert orchestration and establishes convergence guarantees. Extending \citet{jonckheere2023symphonyexpertsorchestrationadversarial}, we show how potential-based updates achieve sublinear regret even when using biased estimates for value functions.

\subsection{Potential-Based Strategies and Regret Guarantees}

To orchestrate expert policies with strong performance guarantees, we leverage tools from online learning, particularly adversarial learning strategies designed to achieve sublinear regret. In particular, we describe how certain potential-based methods for combining expert policies provide regret guarantees with respect to the policy class \(\mathcal{C}(\Pi)\). A key result is that the cumulative regret grows sublinearly with the time horizon \(T\) and the number of experts \(K\), typically of order \(\mathcal{O}(\sqrt{T \ln K})\) for classical strategies.

This section outlines how these techniques apply in our setting, explains their theoretical guarantees, and sets the stage for later extensions to biased or estimated quantities. Our goal is to provide both intuition and formalism for why these strategies are effective for adaptive policy selection, and how they ensure provable performance improvements over time.

The potential-based strategies considered here rely on a function \(\varphi: \mathbb{R} \to [0, +\infty)\).
The initial weights are set as \(q_{1}(k|s) = \frac{1}{K}\)  $\forall s \in \mathcal{S}, k \in [K]$. For \(t \geq 2\), the weights are updated as

\begin{align} \label{eq:weights-update-step}
    q_t(k|s) = \frac{\varphi_t \left( \sum_{h=1}^{t-1} A_{q_h\Pi}(s, k)  \right)}{\sum_{j \in [K]} \varphi_t \left( \sum_{h=1}^{t-1} A_{q_h\Pi}(s, j)  \right)} \;\;\;\; \forall s \in \mathcal{S}, k \in [K]
\end{align}

Several potential functions have been extensively studied in the literature, offering distinct strategies for regret minimization in adversarial learning. We consider the following two strategies (\citet{jonckheere2023symphonyexpertsorchestrationadversarial}):

\begin{itemize}
    \item \textbf{Polynomial Potential:} \[\varphi_t\left(\sum_{h=1}^{t-1} A_{q_h\Pi}(s, k)\right) = \max\left\{\sum_{h=1}^{t-1} A_{q_h\Pi}(s, k) , 0\right\}^p \;\;\;\; \forall s \in \mathcal{S}, k \in [K],\] where \(p \geq 2\),

    \item \textbf{Exponential Potential:}  \[\varphi_t\left(\sum_{h=1}^{t-1} A_{q_h\Pi}(s, k) \right) = \exp\left(\eta_t \sum_{h=1}^{t-1} A_{q_h\Pi}(s, k) \right) \;\;\;\; \forall s \in \mathcal{S}, k \in [K]\].  
   
\end{itemize}
These strategies share the property of controlling regret in the adversarial setting: for all \( T \geq 1 \),
\begin{align}\label{eq:adv_regret}
    \max_{k \in [K]} \sum_{t=1}^T A_{q_t}(k|s) - \sum_{t=1}^T \sum_{j \in [K]} q_t(j|s)\, A_{q_t\Pi}(s,j) 
    = \max_{k \in [K]} \sum_{t=1}^T A_{q_t}(k|s)
    \leq B_{T,K},
\end{align}

This inequality follows from classical results in adversarial online learning, specifically, the framework of regret minimization against an adaptive or adversarial reward sequence, rather than a stationary RL environment. The bound \( B_{T,K} \) represents the worst-case regret of the learner relative to the best fixed action (in hindsight), and depends on both the time horizon \( T \) and the number of alternatives \( K \). Although these guarantees were originally derived outside the RL setting, they provide a robust foundation for reasoning about adaptive policy selection under uncertainty.

This adversarial regret control forms the backbone for deriving Lemma~\ref{lem:jonckheere2023}, which we adapt from \citet{jonckheere2023symphonyexpertsorchestrationadversarial}. For the sake of completeness, we provide a full proof of Lemma~\ref{lem:jonckheere2023} in Appendix~A.

\begin{lemma}[\citealp{jonckheere2023symphonyexpertsorchestrationadversarial}]
\label{lem:jonckheere2023}
Suppose a learning strategy selects weights \(\{q_t\}\) over \(K\) experts, producing a sequence of policies \(\{q_t\Pi\}\). For the strategies described above, the following regret bound holds. For any fixed policy \(q\Pi \in \mathcal{C}(\Pi)\) and any \(T > 1\),
\[
\sum_{t=1}^{T} \Bigl( V_{q\Pi}(s_1) - V_{q_t\Pi}(s_1) \Bigr)
\le \frac{1}{(1 - \gamma)^2} B_{T, K},
\]
where \(B_{T, K}\) can be made sublinear in \(T\) using appropriate strategies.
\end{lemma}

The regret bounds in Table~\ref{tab:convergence_rates} show how different choices of potential function control the term $B_{T,K}$ in a (purely) adversarial learning setting, which in turn translates into performance guarantees in our policy learning scheme. These methods balance exploration and exploitation when combining expert policies, with regret bounds scaling as \(\mathcal{O}(\sqrt{T \ln K})\) when parameters are set optimally.

\begin{table}[t]
\centering
\caption{Summary of regret bounds \( B_{T,K} \) for standard potential-based strategies in adversarial learning. These bounds control the regret of the learner relative to the best fixed decision and underpin the convergence guarantees in Lemma~\ref{lem:jonckheere2023}.}

\label{tab:convergence_rates}
\small
\begin{tabular}{@{}lll@{}}
\toprule
\textbf{Strategy} & \textbf{Potential Form} & \textbf{Convergence Rate} \\ \midrule
Polynomial Potential & \(\varphi_t(x) = \max\{x, 0\}^p, p = 2 \ln K \) & \(\sqrt{6T \ln K}\) \citep{cesa2003potential}\\
Exponential (Fixed) & \(\varphi_t(x) = \exp(\eta x)\) & \(\ln K/\eta + \eta T/2\) \citep{cesa2003potential} \footnotemark \\
Exponential (Varying) & \(\varphi_t(x) = \exp(\eta_t x)\),  \(\eta_t = \frac{1}{M} \sqrt{\ln K/t}\) & \(\sqrt{T \ln K}\) \citep{Auer02} \\ \bottomrule
\end{tabular}
\end{table}

\footnotetext{The policy learning scheme in Section 4, which uses oracle assistance, is equivalent to natural policy gradient ascent with a softmax parametrization (Agarwal et al., 2021, Section 5.3) and achieves a regret bound of $\mathcal{O}(1/T)$. However, replacing oracle calls with estimations introduces significant challenges, and the convergence guarantees do not generally extend to this modified setting}

\subsection{Policy Orchestration
based on Estimated Values}

In most practical settings, we do not have access to oracles or unbiased estimators. Therefore, we extend the strategies described earlier to work without oracle assistance by estimating the advantage functions. To achieve this, we use biased and bounded estimators.

For each given policy $\pi$, starting state $s$, and action $a$, we compute an estimation of the respective value function $Q_\pi(s, a)$, which we will refer to as $\widetilde{Q}_\pi(s, a)$. Using these $Q$-value estimates, we construct estimates of the advantage functions. This scheme can be extended from actions to "super-actions" as defined by expert policies. Details on the construction of the estimators can be found in Section~\ref{sec:implementation_schemes}.

The estimation procedure introduces a source of randomness in the computation of the weights, as they are derived from stochastic updates. To formalize this, let \( \mathcal{F}_{t-1} \) denote the \( \sigma \)-algebra generated by the randomness inherent in the estimation process up to round \( t-1 \). Given this, the weights \( q_t \) are constructed based on these estimates, following the same strategy as in \eqref{eq:weights-update-step} (as defined in \citealp{jonckheere2023symphonyexpertsorchestrationadversarial}), but with the estimated advantage function replacing the true quantities

\begin{align} \label{eq:weights-update-step-estimation}
    q_t(k|s) = \frac{\varphi_t \left( \sum_{h=1}^{t-1} \widetilde{A}_{q_h\Pi}(s, k)  \right)}{\sum_{j \in [K]} \varphi_t \left( \sum_{h=1}^{t-1} \widetilde{A}_{q_h\Pi}(s, j)  \right)} \;\;\;\; \forall s \in \mathcal{S}, k \in [K],
\end{align}
where $\varphi_t$ is one of the potentials introduced in section \ref{sec:learning-strat}. 

\begin{assumption}
\label{lm:est-eps}
For all $t \geq 1$, the estimators $\widetilde{A}_{q_t\Pi}(s,k)$ are $\mathcal{F}_t$-measurable and satisfy the following properties:
\begin{align}\label{lemma:prop1}
\begin{split}
    & \bigl| \widetilde{A}_{q_t\Pi}(s,k) \bigr| \leq \frac{1}{(1-\gamma)} \quad \text{almost surely}, \\
\mbox{and} \quad &
\sum_{k \in [K]} q_t(k|s) \, \widetilde{A}_{q_t\Pi}(s,k) = 0 \quad \text{almost surely},
\end{split}
\end{align}
and on the other hand,
\begin{align}\label{lemma:prop2}
    \Bigl| {\mathbb{E}}\bigl[ \widetilde{A}_{q_t\Pi}(s,k) \mid \gF_{t-1} \bigr] - A_{q_t\Pi}(s,k) \Bigr| \leq \epsilon \quad \text{almost surely}.
\end{align}

\end{assumption}

\begin{theorem}[Control in Expectation]
\label{th:cesaro}
If \eqref{eq:adv_regret} holds for a sequential strategy $\varphi$, and the weights $q_t$ are computed according to \eqref{eq:weights-update-step-estimation} using estimated value functions that satisfy Assumption~\ref{lm:est-eps}, then the stationary policies induced by these weights control the regret with respect to $\mathcal{C}(\Pi)$ as follows:

For all $s_0 \in \gS$ and $T \geq 1$,
\[
V_{q^\star\Pi}(s_0) - \frac{1}{T} \sum_{t=1}^T {\mathbb{E}}\bigl[ V_{q_t\Pi}(s_0) \bigr]
\leq \frac{\epsilon}{1-\gamma} + \frac{B_{T,K}}{ (1-\gamma)^2 T}\,. 
\]
\end{theorem}

See Appendix~\ref{sec:proofthm1} for the proof of Theorem~\ref{th:cesaro}.

\paragraph{Analysis in high probability} 
 Compared to the expected bound of Theorem~\ref{th:cesaro},
the high-probability regret bound adds a factor
$2 \ln(1/\delta)/ \bigl((1-\gamma)^2\sqrt{T} \bigr)$,
which is of the same order of magnitude as the main term
$B_{T,K} / \bigl(  (1-\gamma)^2 T \bigr)$.

\begin{theorem}[high-probability control]
\label{th:HP-estim}
If \eqref{eq:adv_regret} holds for a sequential strategy $\varphi$, then the stationary policies based on the weights $q_t$ defined defined in~\eqref{eq:weights-update-step-estimation}
control the regret w.r.t.\ $\gC(\Pi)$ as follows:
for all $\delta \in (0,1)$,
for all $T \geq 1$,
for all~$s_0$,
with probability at least $1-\delta$,
\begin{align*}
\smash{V_{q^\star\Pi}(s_0) - \frac{1}{T} \sum_{t=1}^T V_{q_t\Pi}(s_0)} \leq
\frac{\epsilon}{1-\gamma} + \frac{B_{T,K}}{ (1-\gamma)^2 T} + \frac{2 \ln(1/\delta)}{(1-\gamma)^2\sqrt{T}}\,.
\end{align*}
\end{theorem}

The proof of Theorem~\ref{th:HP-estim} is provided in Appendix~\ref{sec:proofthm2}.

The proofs of Theorems~\ref{th:cesaro} and \ref{th:HP-estim} relies on standard tools such as discounted visitation distributions and the performance difference lemma (\citealp[Lemma~6.1]{KL02}); see Appendix~\ref{sec:thmProofs} for full details.

\begin{remark}
    All what stated so far can be easily generalised to reward functions taking values in a range $[-M, M]$.
\end{remark}

\subsection{Ensuring Assumption~\ref{lm:est-eps} via Temporal Difference Learning}\label{sec:tabular-adv}
To apply the above regret guarantees in practice, we must ensure that the advantage estimates $\hat{A}_{q_t\Pi}(s, a)$ satisfy Assumption~\ref{lm:est-eps}. We show that temporal-difference (TD) learning provides such estimates under mild conditions.

\paragraph{Temporal-Difference Updates in Tabular Setting} Algorithm~\ref{alg:1} in Appendix~\ref{sec:alg:est} illustrates how a tabular approximation of the advantage function can be maintained and updated using a one-step temporal-difference (TD) method. At each time step, upon observing a transition \((s_\tau, k_\tau, a_\tau, r_\tau)\), the estimate \(\widetilde{Q}_{q\Pi}(s_\tau, k_\tau)\) is updated according to the following rule:
\begin{align}\label{eq:update_rule}
    \widetilde{Q}_{\pi, \tau+1}(s_\tau, a_\tau) &= (1 - \alpha)\widetilde{Q}_{\pi, \tau}(s_\tau, a_\tau) + \alpha \left( r_\tau + \gamma \widetilde{Q}_{\pi, \tau}(s_{\tau+1}, a'_{\tau+1}) \right), \\
    \widetilde{Q}_{\pi, \tau+1}(s, a) &= \widetilde{Q}_{\pi, \tau}(s, a) \quad \forall\, s, a \neq s_\tau, a_\tau.
\end{align}

The value estimate is then obtained by mixing over \(k \in [K]\) with \(q(k \mid s)\), and the advantage is defined accordingly.

\paragraph{Finite-time bias bound for TD learning}

We consider a fixed stationary policy $\pi$ over a finite state space $S$ and action space $A$, with stationary state-action distribution $p_\pi(s, a) := d_\pi(s) \pi(a | s)$, where $d_\pi$ is the stationary state distribution. Let $\alpha > 0$ be the constant step size, and $\gamma \in (0,1)$ the discount factor.

Define the bias at iteration $\tau$ as the difference between the estimated and true Q-values:
\[
b_\tau(s, a) := \widetilde{Q}_{\pi,\tau}(s, a) - Q_\pi(s, a).
\]

We assume \( p_\pi(s,a) > 0 \) for all reachable pairs \((s,a)\), and define
\[
\kappa := \alpha (1 - \gamma) \min_{s,a} p_\pi(s,a),
\]
where \(\alpha > 0\) is the stepsize and \(\gamma \in (0,1)\) is the discount factor.

We also define
\[
E := \exp\left(\frac{C}{(1 - \kappa)(1 - \rho)}\right),
\]
where \( C > 0 \) is a constant, and \( \rho \in (0,1) \) controls the decay rate of the per-iteration perturbations.

\begin{lemma}[Bias contraction bound under stationary policy]\label{lemma:adv-bias-bound}
Under the above assumptions, the expected bias satisfies the uniform bound:
\[
\left| \mathbb{E}\left[ \widetilde{A}_{\pi, \tau}(s, a) \right] - A_\pi(s, a) \right| \leq 2 \| \mathbb{E}[b_\tau] \|_\infty \leq (1 - \kappa)^\tau \cdot 2E \cdot \| \mathbb{E}[b_0] \|_\infty.
\]
In particular, the bias converges geometrically to zero over all reachable state-action pairs.
\end{lemma}

We prove this Lemma in Appendix \ref{sec:proof_lemma}.

\paragraph{Corollary}
Under standard conditions (namely, bounded rewards, geometric ergodicity, constant step size), temporal difference (TD) learning produces advantage estimates $\widetilde{A}_\pi(s,a)$ satisfying Assumption~\ref{lm:est-eps} with bias $\epsilon_\tau$ converging to $0$ at a geometric rate. This justifies their use in Theorems~\ref{th:cesaro} and~\ref{th:HP-estim}.

\paragraph{Remark}
If there exist $(s, a)$ such that $p_\pi(s, a) = 0$, the result still holds for all reachable pairs, and $\min_{s,a} p_\pi(s,a)$ should be understood as taken over those pairs. Unreachable pairs are not visited and thus irrelevant for learning.

\paragraph{Related Work and Comparison}

The bias of temporal difference (TD) learning has been analyzed primarily under linear function approximation and constant step sizes. Classical works such as \citet{tsitsiklis1997analysis} and \citet{borkar2000ode} focus on asymptotic convergence using stochastic approximation, but do not provide finite-time bias guarantees. Recent finite-time analyses by \citet{SrikantYing2019}, \citet{Dalal2018TD}, and \citet{Bhandari2018} derive error and convergence rates under stationary or Markovian sampling, but they do not isolate explicit bias bounds, while \citet{KondaTsitsiklis2004} study asymptotic bias under constant step sizes. Off-policy methods like emphatic TD \citep{Yu2015ETD} address distribution mismatch but remain asymptotic.

Our contribution complements this line of work by providing an explicit finite-time bias bound for TD learning with constant step sizes, under potentially non-stationary sampling. We derive a bound of the form
\[
\|x_{\tau} - x^*\| \leq \rho^\tau \|x_0 - x^*\| + C_0 \alpha + \sum_{s=0}^{\tau-1} \rho^s \delta_{\tau-s},
\]
where $\delta_t$ measures perturbations from sampling shifts. Our bound separates contraction, step-size bias, and perturbation effects, offering a flexible tool for analyzing learning under changing environments. We build on standard contraction arguments, applying them recursively to track error accumulation under non-stationary conditions.

A more detailed discussion of related work and how our results compare can be found in Appendix~\ref{appendix:related-work-lemma}.

\section{Implementation schemes: Policy Updates and Advantage Computation}\label{sec:implementation_schemes}
We develop and analyze policy learning methods for large-scale expert orchestration in reinforcement learning (RL). Our goal is to design algorithms that can dynamically combine expert policies to improve decision-making in complex environments. Crucially, we aim for methods that not only scale to large state and action spaces but also preserve interpretability, which is essential in domains where transparency and trust are critical.

The learning process is structured around two key steps: (i) updating the weights $q_t(\cdot \mid s)$ assigned to each expert policy $\pi_k \in \Pi$ based on an estimated advantage function, and (ii) evaluating or estimating this advantage function under the updated policy to inform future updates. We explore two complementary approaches for realizing this process: a \emph{tabular} method for smaller state-action spaces and a \emph{neural network} method for high-dimensional or continuous environments. Both approaches follow the same general framework, but differ in how they represent and approximate the advantage function.

In addition to these, we propose a novel \emph{NN-based policy learning approach}, where not only the advantage function but also the policy itself is directly parameterized by a neural network. This method enables continuous generalization across the state space and removes the need for explicit tabular representations, providing a scalable and flexible solution for complex decision-making tasks. Together, these methods demonstrate how advantage-based orchestration can adapt to different problem regimes and open the door to principled learning in large-scale settings.

From now on, we will use \( q \) to denote a general set of weights, representing any policy that maps each state \( s \in \mathcal{S} \) to a probability distribution \( q(\cdot \mid s) \) over \([K]\). In the description of routines and algorithms, we will use \( q_t \) to refer specifically to the set of weights learned at step \( t \) of the main algorithm, where \( t \) indexes the overall learning iterations of the policy.  

Within each learning step \( t \), we employ subroutines to estimate the advantage function. In this framework, we use \( \tau \) to index the discrete time steps within these subroutines, where observations and updates occur for estimating the Q-values and advantage function. At each step \( \tau \), the agent observes a transition in the environment, characterized by the state \( s_\tau \), action \( a_\tau \), action index \( k_\tau \), and the resulting reward \( r_\tau \). This temporal indexing ensures the sequential updating of the Q-value estimate, \( \widetilde{Q}_{q_t \Pi}(s_\tau, k_\tau) \), facilitating the computation of the value estimate through a mixture over action indices \( k \in [K] \). The advantage function is then derived from these updated estimates.

We begin by describing the advantage estimation procedures in both tabular and neural network settings, then detail the two policy learning approaches.

\subsection{Advantage Function Estimation}

Accurately estimating the advantage function is essential for guiding policy improvement, whether in simple tabular settings or in complex, high-dimensional environments.

\begin{remark} \label{remark:A_contruction}
    In both estimation procedures described below, we do not directly estimate the advantage function. Instead, we first estimate the Q-function and then compute the advantage using the following formula:
    \begin{align*}
        \widetilde{V}_{q\Pi}(s) &= \sum_{k \in [K]}q(k|s) \widetilde{Q}_{q\Pi}(s, k) \\
        \widetilde{A}_{q\Pi}(s, k) &= \widetilde{Q}_{q\Pi}(s, k) - \widetilde{V}_{q\Pi}(s)
    \end{align*}
for each $s\in \mathcal{S}$.
\end{remark}

\paragraph{Tabular} See Section~\ref{sec:tabular-adv} for details on the tabular estimation of $A_{q_t\Pi}(s,k)$ (and performance guarantees).

\paragraph{Neural Network}\label{sec:adv_learning_NN}

In environments with high-dimensional or continuous state spaces, maintaining an explicit table for $Q_{q\Pi}(s,k)$ or $A_{q_t\Pi}(s,k)$ becomes impractical. Instead, we employ a neural network (NN) (see Figure \ref{fig:neural_network}) to approximate the advantage function. In Algorithm~\ref{alg:2} (Appendix~\ref{sec:alg:est}), we use a DQN-based architecture adapted to combine $K$ experts, with modifications to the output layer to reflect the $(s,k)$ structure. 

An alternative approach is Double DQN \citep{vanhasselt2015deepreinforcementlearningdouble}, which mitigates Q-value overestimation by using separate networks for action selection and value estimation (see Algorithm~\ref{alg:doubleDQN}, Appendix~\ref{sec:alg:est}).

\begin{figure}[h]
    \centering
    \begin{tikzpicture}[
        node distance=0.9cm and 1.2cm,
        every node/.style={draw, circle, minimum size=0.8cm, inner sep=0pt},
        input/.style={fill=blue!20},
        hidden/.style={fill=green!20},
        output/.style={fill=red!20},
        textnode/.style={draw=none, fill=none, font=\scriptsize}
        ]
        
        \node[input] (I1) {};
        \node[input, below=0.6cm of I1] (I2) {};
        \node[input, below=0.6cm of I2] (I3) {};

        \node[hidden, right=of I2] (H11) {};
        \node[hidden, above=0.6cm of H11] (H12) {};
        \node[hidden, below=0.6cm of H11] (H13) {};
        
        \node[hidden, right=of H11] (H21) {};
        \node[hidden, above=0.6cm of H21] (H22) {};
        \node[hidden, below=0.6cm of H21] (H23) {};
        
        \node[output, right=of H21] (O1) {};

        \node[textnode, left=0.4cm of I2] {\textbf{Input}};
        \node[textnode, right=0.4cm of O1] {\textbf{Output (linear)}};
        
        \node[draw=none, fill=none, fit=(H12) (H22), textnode, yshift=0.7cm] {\textbf{Fully connected layers + ReLU}};

        \foreach \i in {1,2,3}
            \foreach \j in {11,12,13}
                \draw[->] (I\i) -- (H\j);

        \foreach \i in {11,12,13}
            \foreach \j in {21,22,23}
                \draw[->] (H\i) -- (H\j);

        \foreach \i in {21,22,23}
            \draw[->] (H\i) -- (O1);
    \end{tikzpicture}
    \caption{Neural network \(\mathcal{N}_{\theta}\) with parameters \(\theta \in \Theta\): two fully connected layers with ReLU activation, followed by a linear output layer.}
    \label{fig:neural_network}
\end{figure}
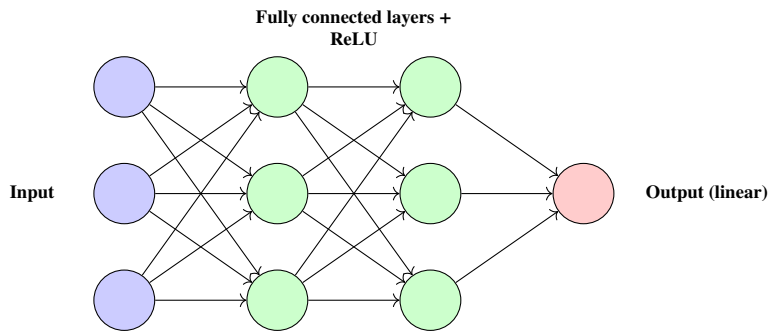

\medskip

\subsection{Tabular Approach for Policy Learning}

In the tabular case, we represent the policy weights $q_t(k\mid s)$ directly in a table for every state-expert pair $(s,k)$, and similarly store or update the advantage values in a tabular structure.  We recall from the previous section that the mixing weights $q_t(\,\cdot\,|s)$ can be updated by leveraging the adversarial-learning strategies (e.g., polynomial or exponential potentials) to guarantee sublinear regret (see Table \ref{tab:tabular-policy}).  The advantage function, in turn, can be estimated either \textbf{(i)} fully in tabular form, or \textbf{(ii)} using a neural network while still employing the tabular framework for the policy weights.

\begin{table}[ht]
\centering
\caption{Tabular policy learning approach referencing expert strategies from the previous section. The advantage function $\widetilde{A}_{q_t \Pi}$ can be estimated by a tabular update or by a neural-network approximation (NN), as indicated in the rightmost column.}
\label{tab:tabular-policy}
\begin{tabular}{@{}lcc@{}}
\toprule
\textbf{Step} & \textbf{Update Mechanism} & \textbf{Advantage Estimation} \\
\midrule
1. \textbf{Initialize} & 
\begin{minipage}[t]{0.44\linewidth}
\raggedright
\emph{Set initial weights} $q_0(k\mid s)$ for all $(s,k)$, e.g., uniform over $k$. 
\end{minipage}
&
\begin{minipage}[t]{0.24\linewidth}
\raggedright
\textbf{Tabular:} \\
Initialize $\widetilde{Q}_{\,q_0\Pi,0}(s,k)=0.$ \\[4pt]
\textbf{NN:} \\
Initialize network weights.
\end{minipage}
\\
\midrule
2. \textbf{Policy update} & 
\begin{minipage}[t]{0.44\linewidth}
\raggedright
\emph{Use a strategy of choice (polynomial potential, exponential potential) to re-weight experts: $q_{t+1}(\,\cdot\,|s)$ based on $\left(\widetilde{A}_{q_h \Pi}(s,\cdot)\right)_{h\le t}$.}
\end{minipage}
&
\\
\midrule
3. \textbf{Advantage update} & 
\begin{minipage}[t]{0.44\linewidth}
\raggedright
\emph{Collect new data $(s_{\tau}, k_{\tau}, a_{\tau}, r_{\tau})_{\tau \le H}$ by applying $q_{t+1}$ and update advantage estimates.}
\end{minipage}
&
\begin{minipage}[t]{0.24\linewidth}
\raggedright
\textbf{Tabular} or \textbf{NN} (see Algorithms \ref{alg:1} and \ref{alg:2})

\end{minipage}
\\
\bottomrule
\end{tabular}
\end{table}

\subsection{Another approach: Neural Network for Policy Learning} \label{sec:NN_policy_updates}

In contrast to the previous approach, where only the advantage function might be approximated via a neural network while the policy remains tabular, we now also model the policy using a neural network. This leads to a more scalable solution in which the policy is directly parameterized and computed by the network. 

A central motivation for this shift is the limitation of the tabular approach: updating the expert mixture weights requires aggregating advantage estimates from all previous steps $(\widetilde{A}_{q_h\Pi})_{h \le t}$. While feasible in small problems, this quickly becomes impractical in high-dimensional or long-horizon settings, as it would demand storing and reprocessing a growing set of advantage values for each state-expert pair. To overcome this bottleneck, we adopt a more involved architecture where the advantage is maintained online by a critic network, removing the need for explicit historical storage.

Specifically, we adopt an actor-critic framework: the critic is represented by a neural network $\mathcal{N}_{\theta}$ that estimates state-dependent advantages, while the actor is a separate network $\mathcal{M}_{\phi}$ that outputs a probability distribution over the $K$ experts.

A key advantage of this method is that the policy is updated continuously and generalizes across the state space, removing the need to store or access a tabular representation. However, we note that both $\mathcal{N}_{\theta}$ and $\mathcal{M}_{\phi}$ are implemented as relatively simple feedforward neural networks. These architectures do not incorporate domain-specific inductive biases, such as attention mechanisms, relational reasoning, or graph-based structure, which could potentially improve learning efficiency and generalization in structured environments. While sufficient for our current setting, future work may benefit from exploring more expressive or specialized architectures tailored to the combinatorial and graph-structured nature of the expert selection problem.

\paragraph{Comparison with Existing Neural Policy Gradient Methods}

Neural policy gradient (NPG) and actor-critic methods are widely used in reinforcement learning for high-dimensional decision problems \citep{sutton2000policy, Kak01, schulman2015trust}. While our approach maintains the actor-critic structure, it diverges from conventional policy gradient methods in several key respects.

Most notably, we replace gradient-based policy updates with a potential-based scheme that aggregates temporal advantage estimates. This not only removes the need for direct gradient computation, offering increased robustness to biased value estimates and greater stability in non-stationary settings, but also resolves the storage issue inherent in the tabular approach. In traditional potential-based orchestration, policy weights at time $t$ depend on the cumulative sequence of past advantage values, which becomes infeasible to store as $t$ grows. By contrast, our critic network $\mathcal{N}_{\theta}$ provides an online approximation of the advantage function, ensuring that the actor $\mathcal{M}_{\phi}$ can be updated incrementally without retaining the full advantage history.

Furthermore, our actor network outputs discrete distributions over structured decisions, rather than logits over primitive actions, enabling more controlled and interpretable learning.

A detailed technical comparison with related methods is provided in Appendix~\ref{app:pg-comparison}.

\begin{algorithm}[H]
\caption{Neural Network Policy Learning}
\label{alg:3}
\begin{algorithmic}[h]
\STATE \textbf{Input:} state space $\mathcal{S}$, action space $\mathcal{A}$, $K$ experts $\{\pi_1, \pi_2, \ldots, \pi_K\}$, transition kernel $\mathcal{T}: \mathcal{S}\times \mathcal{A} \mapsto \mathcal{P}(\mathcal{S})$, reward function $\mathcal{R}: \mathcal{S}\times \mathcal{A} \mapsto [0, 1]$, potential functions $\left(\varphi_t: \mathbb{R} \mapsto [0, \infty]\right)_{t\le T}$ (see section \ref{sec:learning-strat}), a parameter space $\Phi$, a loss function for the actor $\mathcal{L}_{\rm{actor}}:\Phi \mapsto \mathbb{R}$ (e.g., KL-divergence), number of learning steps $T$, learning rate decay factor and discount factor $\lambda, \gamma \in (0, 1)$, and an \textbf{advantage simulator} for estimating $\widetilde{A}_{q\Pi}$. 
\STATE \textbf{Initialize:}  State $s_0$, starting learning rate $\alpha_0$, Neural network $\mathcal{M}_{\phi_0}(s)$, with parameters $\phi_0 \in \Phi$, Advantage simulator (e.g., Double DQN) with replay buffer $\mathcal{D}$.
\FOR{episode $t= 1$ to $T$}
    \STATE Reset environment
    \FOR{$\tau = 1$ to $H$}
    \STATE Compute expert distribution $\mathcal{M}_{\phi}(s_\tau) $
    \STATE Observe $s_\tau$, select and execute action $k_\tau \sim \mathcal{M}_{\phi}(s_\tau)$, receive reward $r_\tau$, and observe next state $s_{\tau+1}$.
    \STATE Store transition $(s_\tau, k_\tau, r_\tau, s_{\tau+1})$ in buffer $\mathcal{D}$.  
        \STATE Update the advantage simulator (see Section~\ref{sec:adv_learning_NN}) 
        \STATE Set $ s_\tau = s_{\tau+1}$
    \ENDFOR
    
    \STATE \textbf{Distribution network update:}
    \STATE \quad Sample a minibatch $B_e$ from $\mathcal{D}$
    \STATE \quad Query the advantage simulator for estimated advantages: 
    \[
    \widetilde{A}_{q\Pi}(B_e, \cdot) \leftarrow \text{Simulator}(B_e)
    \]
    \STATE \quad Compute new expert distribution:
    \[
       \mathcal{M}_{\rm{target}}(\cdot \mid B_e) 
       = \varphi_t
        \left( \,  \sum_{h=0}^{t-1}\widetilde{A}_{q_h\Pi}(B_e, \cdot)  \right)
    \]
    \STATE \quad Compute distribution update loss $\mathcal{L}_{\mathrm{actor}} (\phi_t)= \ell\bigl(\mathcal{M}_{\rm{target}},\mathcal{M}_{\phi_t})$
    \STATE \quad Backpropagate and update $\phi_t$:
\begin{itemize}
    \item  Compute gradient of loss w.r.t. network parameters: $\nabla_{\phi_t} \mathcal{L}_{\text{actor}}(\phi_t)$.
     \item Update online network parameters using Adam optimizer:
      \[
    \phi_{t+1} = \phi_{t} - \alpha_t \nabla_{\phi_t} \mathcal{L}_{\text{actor}}(\phi_t),
    \]
\end{itemize}
\STATE Decay learning rate: $\alpha_{\tau+1} = \alpha_{\tau} \cdot \lambda$.
     
\ENDFOR
\STATE Set $q_T(\cdot \mid s) = \mathcal{M}{\phi_T}(s)$ for all $s \in \mathcal{S}$
\STATE \textbf{Output:} Learned expert mixture $q_T\Pi$.
\end{algorithmic}
\end{algorithm}

\textbf{Training details:}

\begin{itemize}
    \item The neural network $\mathcal{M}_{\phi}$ is trained using the Adam optimizer \citep{kingma2014adam} with a decaying learning rate. The learning rate $\alpha_t$ is reduced after each update step by a decay factor $\lambda$ to improve stability during training.
    \item The loss function $\mathcal{L}_{\text{actor}}$ for updating the distribution network is computed by comparing the updated expert mixture $\mathcal{M}_{\text{target}}$ (derived from the estimated advantage) to the current policy distribution $\mathcal{M}_{\phi_t}$:
    \[
    \mathcal{L}_{\text{actor}} = \ell\left(\mathcal{M}_{\text{target}},\mathcal{M}_{\phi_t}\right),
    \]
    where $\ell$ is a suitable distance metric (e.g., Kullback-Leibler divergence or cross-entropy).
    
    \item The gradients of $\mathcal{L}_{\text{actor}}$ with respect to $\phi_t$ are computed via backpropagation, and the parameters $\phi_t$ are updated using the Adam optimizer. The update rule is:
    \[
    \phi_{t+1} = \phi_{t} - \alpha_t \nabla_{\phi_t} \mathcal{L}_{\text{actor}},
    \]
    where $\alpha_t$ is the learning rate at step $t$.
\end{itemize}

\noindent
In practice, this neural-network policy approach allows for a much more compact representation of large or continuous state spaces, at the cost of introducing additional approximation errors and the need for suitable hyperparameter tuning (e.g., learning rate, batch size, architecture). Once the policy network converges, we obtain a function $s\mapsto ({q}(1\mid s), \dots, {q}(K\mid s))$ that orchestrates the experts effectively based on state features.
\\\\
In Algorithm \ref{alg:3}, we present the procedure described above in detail.

\begin{remark}
    In the algorithms described, we directly use the input expert policies or the learned distribution $\mathcal{M}_{\phi}$ for selecting actions. However, this can be replaced by an $\epsilon$-greedy strategy, where actions are sampled according to the given policy with probability $1 - \epsilon$ and chosen randomly with probability $\epsilon$. The exploration rate $\epsilon$ can be decayed over time to balance exploration and exploitation.
\end{remark} 

\noindent
In summary, the tabular and neural network approaches provide two scalable strategies for expert orchestration. The tabular method offers simplicity and exactness in small-scale settings, while the neural network method generalizes across large or continuous state spaces, trading off approximation error for scalability. Together, they illustrate how advantage-based policy learning can be adapted to different problem regimes.

\section{Stochastic matching model}\label{sec:stoch_model}
We study a discrete-time stochastic matching problem motivated by applications such as online marketplaces, supply chain logistics, and organ exchange programs. The system consists of $I$ item classes, each representing a queue with finite capacity $L$. Items arrive over time, and at each decision point, the system can either match a new or existing item with a compatible counterpart, store it for future matching, or discard it if no feasible option exists.

The matching possibilities are encoded by an undirected compatibility graph, where an edge between classes $i$ and $j$ indicates that items from these classes can be paired. Each class $i$ has an arrival rate $\lambda_i$, and items may leave or relocate at rates $\mu_i$ and $\nu_i$, respectively. To model these stochastic dynamics in discrete time, we apply a uniformization procedure that bundles arrival, departure, relocation, and idle events under a single time-homogeneous process.

\paragraph{Formal details}
The full mathematical formulation, including the transition kernel, event probabilities, and reward computation, is provided in Appendix~\ref{appendix:matching}. This includes the definition of the state space $\gS$, the action space $\gA$, and the exact reward components used in our simulations.

\paragraph{Decision and reward structure}
At each time step, the learner observes the current queue state and event type (arrival, departure, relocation) and selects an action $a_t$: either match two items, place the new item into its queue, or trash it. Rewards balance positive gains for successful matches against negative costs for maintaining queues or relocating items. For example, in organ exchange, matches are weighted by clinical priority and compatibility, while unused organs incur wastage costs.

\paragraph{Expert policies}
We define four expert policies as follows:
\begin{itemize}
    \item The first expert policy, $\pi_1$, is called \emph{match the longest}. If at least one match is possible, this policy always selects the class with the largest number of items in its queue. In the case of a tie, the policy breaks it based on the payoffs.

    \item The second expert policy, $\pi_2$, is of the \emph{edge-priority} type. It selects matches according to a priority order defined on the edges of the compatibility graph. If at least one match is possible, $\pi_2$ chooses the match that leads to the largest payoff. Ties are broken based on queue lengths.

    \item The third expert policy, $\pi_3$, also follows a greedy approach like $\pi_2$, but with an additional restriction: it only applies the greedy policy to the items belonging to a specific set of compatible classes $P(\pi_3)$. $\pi_3$ ignores items that do not belong to this set. This policy will be employed in the donor exchange example (\ref{sec:organ_exchange}), where it might be beneficial to overlook matches with a low-urgency class in favor of waiting for a higher reward match.

    \item The fourth, and final, expert policy, $\pi_4$, is the uniform policy: it randomly selects a match among the available ones.

\end{itemize}

If no match is possible in any of the above policies, in the events of an arrival or a relocation, they proceed as follows: if the maximum allowed queue length $L$ has not been reached, the item is added to the queue; if the queue is full, the item is discarded (``trashed'').

\paragraph{Example models}
We illustrate this general framework using two representative examples, depicted in Figures~\ref{fig:diamond_graph} and~\ref{fig:org_don_graph}.

\begin{itemize}
    \item \textbf{Diamond graph:}
The diamond graph in Figure~\ref{fig:diamond_graph} represents a symmetric four-node matching environment. Each node corresponds to a class of items with specific arrival rates, and edges encode matchable pairs along with associated rewards. This setup fits directly into our general stochastic matching framework: the state space is defined by queue lengths at each node, the action space comprises match decisions or queue management actions, and transitions reflect the stochastic arrival and match dynamics.

At each time step, the system observes an event (e.g., arrival) and selects an action (e.g., match two items or queue an incoming item). The reward function assigns fixed values to each edge: for instance, the edge $(2,4)$ yields a high reward of $200$, incentivizing that match when feasible. Expert policies such as ``match the longest'' or ``greedy payoff'' can exhibit short-term gains but may ignore future match opportunities. This controlled setting allows us to evaluate how our orchestration strategy learns to balance such trade-offs by adaptively combining policies to improve long-term value.

    \begin{figure}[h]
	\centering
	\begin{tikzpicture}[scale=1]
		\tikzstyle{roundnode} = [shape=circle, draw, fill=blue!10, minimum size=0.01cm]
		
		\node[roundnode] (A) at (0, 0) {2};
		\node[roundnode] (B) at (-2, -1.5) {1};
		\node[roundnode] (C) at (2, -1.5) {4};
		\node[roundnode] (D) at (0, -3) {3}; 
		
		\draw (A) -- (B);
		\draw (A) -- (C);
		\draw (A) -- (D); 
		\draw (B) -- (D);
		\draw (C) -- (D);
		
		\draw (A) -- (B) node[midway, above left] {$g_{1,2} = 10$};
		\draw (A) -- (C) node[midway, above right] {$g_{2,4}= 200$};
		\draw (A) -- (D) node[midway, right] {$g_{2,3}= 50 $};
		\draw (B) -- (D) node[midway, below left] {$g_{1,3}= 1$};
		\draw (C) -- (D) node[midway, below right] {$g_{3,4}= 20$};
		
		\draw[->] (0, 1) -- (A) node[midway, right] {$\lambda_2= 0.225 $};
		\draw[->] (-4, -1.5) -- (B) node[midway, above ] {$ \lambda_1 = 0.125 $};
		\draw[->] (4, -1.5) -- (C) node[midway, above] {$\lambda_4= 0.05 $};
		\draw[->] (0, -4) -- (D) node[midway, right] {$\lambda_3= 0.15 $};
	\end{tikzpicture}%
	\caption{Diamond network with four nodes (labeled 1--4).}
	\label{fig:diamond_graph}
\end{figure}
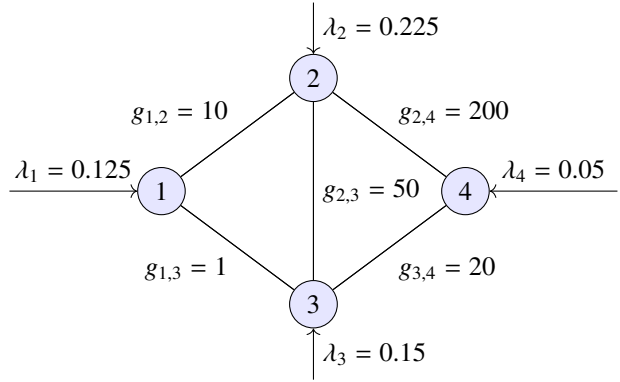
    \item \textbf{Organ exchange model:}
Figure~\ref{fig:org_don_graph} shows a compatibility graph inspired by organ exchange programs, where recipients (colored circles) and donors (green rectangles) are grouped by blood type and urgency level. Edges represent feasible medical matches. This environment is modeled as a stochastic matching problem with heterogeneous queues, rare classes, and urgency-based transitions. 

The state includes the number of patients at each node; actions involve selecting feasible matches, queuing, or discarding items. Transitions are influenced by arrivals, departures, and urgency escalations, modeled via relocation to higher urgency states. Rewards are weighted by urgency and penalize discards or delayed matches. For example, relocating a patient from urgency level 0 to 1 incurs a smaller penalty than losing a level-2 patient to departure.

This example highlights the importance of policy adaptivity: naive heuristics may favor frequent, low-urgency matches, while the orchestrator can learn to prioritize high-reward opportunities even if they are rare. It demonstrates how our framework accommodates both fairness and efficiency in high-stakes, structured decision problems. For a detailed explanation of this model see Section \ref{sec:organ_exchange}.

    \begin{figure}[t]
    \centering
    \scalebox{0.85}{
    \begin{tikzpicture}[scale=0.1, every node/.style={draw, circle, minimum size=2.5mm, font=\scriptsize}, donor/.style={rectangle, fill=green!50!black, font=\scriptsize}]
        \node[donor] (0) {0};
        \node[donor, right=1.25cm of 0] (1) {1};
        \node[donor, below left=3cm and 0.3cm of 0] (2) {2};
        \node[donor, below right=3.5cm and 2.5cm of 1] (3) {3};

        \node[fill=red!70] (4) [above left=0.9cm and 0.6cm of 0] {4};
        \node[fill=orange!70] (5) [below=0.3cm of 4] {5};
        \node[fill=yellow!70] (6) [below=0.3cm of 5] {6};

        \node[fill=red!70] (9) [above right=0.6cm and 0.6cm of 1] {9};
        \node[fill=orange!70] (8) [above=0.3cm of 9] {8};
        \node[fill=yellow!70] (7) [above=0.3cm of 8] {7};

        \node[fill=orange!70] (11) [left=0.9cm of 2] {11};
        \node[fill=red!70] (10) [below=0.3cm of 11] {10};
        \node[fill=yellow!70] (12) [above=0.3cm of 11] {12};

        \node[fill=red!70] (13) [left=0.9cm of 3] {13};
        \node[fill=orange!70] (14) [above=0.3cm of 13] {14};
        \node[fill=yellow!70] (15) [above=0.3cm of 14] {15};

        \foreach \i/\j in {0/4, 0/5, 0/6, 1/7, 1/8, 1/9, 2/10, 2/11, 2/12, 3/13, 3/14, 3/15, 9/0, 8/0, 7/0, 13/0, 13/1, 13/2,14/0, 14/1, 14/2, 15/0, 15/1, 15/2, 12/0, 11/0, 10/0}
            \draw[thick] (\i) -- (\j);

        \begin{scope}[on background layer]
            \node[draw, ellipse, fit=(0)(4)(5)(6), inner sep=3pt, label={0}] {};
            \node[draw, ellipse, fit=(1)(7)(8)(9), inner sep=3pt, label={A}] {};
            \node[draw, ellipse, fit=(2)(10)(11)(12), inner sep=3pt, label={B}] {};
            \node[draw, ellipse, fit=(3)(13)(14)(15), inner sep=3pt, label={AB}] {};
        \end{scope}

        \draw[->] (1) -- ++(20,0) node[right, draw=none] {donor};
    \end{tikzpicture}
    }
    \caption{Compatibility graph in the organ donation example. Donors (green rectangles) connect to recipients (colored circles) according to blood type and urgency.}
    \label{fig:org_don_graph}
\end{figure}
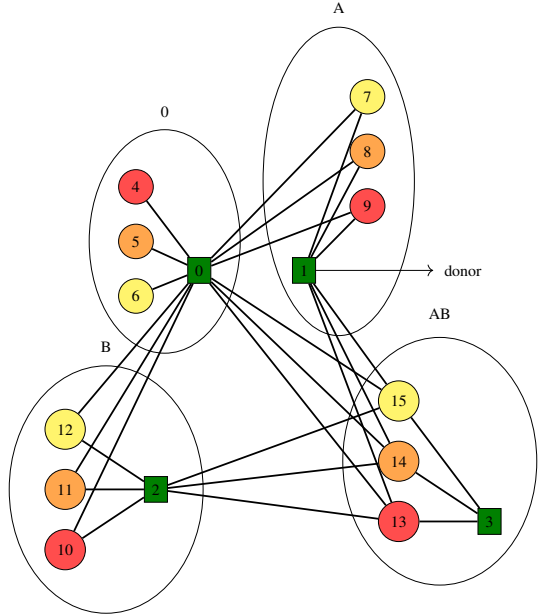
\end{itemize}

These two models exemplify how our general orchestration framework can be applied to both synthetic and realistic domains. The diamond graph offers a controlled setting for testing convergence and performance guarantees, while the organ exchange model highlights the framework’s ability to handle real-world constraints like fairness, urgency, and sparse compatibility. We provide experimental validation in Section \ref{sec:simulations}.

\section{Simulations}\label{sec:simulations}
In this section, we systematically evaluate the effectiveness of our proposed orchestration strategies through a series of simulation experiments. Our goal is to demonstrate how adaptive expert policy orchestration improves learning speed, policy quality, and scalability across representative stochastic matching problems.

We focus on two complementary models:
\begin{itemize}
    \item Diamond graph (Figure \ref{fig:diamond_graph}): a controlled, symmetric environment designed to test and compare learning dynamics under known conditions.
    \item Organ exchange network (Figure \ref{fig:org_don_graph}): an heterogeneous environment capturing the complexity and high stakes of medical decision-making.
\end{itemize}

For each setting, we compare our methods against individual expert baselines as well as standard reinforcement learning algorithms, such as QL and Double DQN. We assess performance based on average value improvement over time, convergence behavior, and robustness across independent runs.

The results demonstrate that combining interpretable expert policies via RL-based orchestration leads to faster convergence and higher performance than both standalone experts and traditional RL methods, particularly in complex or large-scale environments.

\paragraph{Learning Schemes} The experiments in the following sections involves three different methodologies:
\begin{enumerate}
    \item Direct Tabular Policy Learning with Tabular Advantage Learning (TD Learning) \ref{sec:tabular-adv};
    \item Direct Tabular Policy Learning with Neural Network Advantage Learning (NNAL)  \ref{sec:adv_learning_NN};
    \item Neural Network Policy Learning with Neural Network Advantage Learning (NNAL) \ref{sec:NN_policy_updates}.
\end{enumerate}

\paragraph{Policy Update Strategies}
We consider the three potential-based learning schemes introduced in section \ref{sec:learning-strat}:
Polynomial potential (PP), Exponential potential with a \emph{fixed} learning rate (EP-C), Exponential potential with \emph{time-varying} learning rates (EP-T).

Below, we first describe the experimental setup and then present detailed quantitative results.

\subsection{Diamond Graph}

The first set of experiments is performed on the diamond graph depicted in Figure~\ref{fig:diamond_graph}. This simple network consists of four nodes, structured to test and compare different learning strategies. All departure and relocation probabilities are set to zero for simplicity. The arrival rates and rewards are specified in Figure~\ref{fig:diamond_graph}, while the maximum queue capacity is set to $L=5$. All model specifications are also detailed in tables \ref{tab:diamond_params} and \ref{tab:global_params_diamond} in Appendix \ref{sec:sim:setting1}.

\paragraph{Performance Criterion} We compare the performance of the algorithms in terms of the values of the learned policies averaged over all possible initial states when starting with an empty system ($V_{{q}_{t} \Pi}(\mu_0)$).

In order to obtain robust estimates, we perform $N$ independent runs of the learning process, each indexed by $n$. For each $t$, the value of the learned policy $q_t$ is averaged over these $N$ independent runs, i.e.,

\begin{align}\label{eq:avg_value}
        t \;\longmapsto\; \frac{1}{N} \sum_{n=1}^N V_{{q}_{t, n} \Pi}(\mu_0),
\end{align}
where $\widetilde{q}_{t, n}$ denotes the policy learned during the $n$-th run. The shaded areas around each curve represent $\pm 2$ standard errors, which are computed by taking the variability across the $N$ different learned policies $\{q_{t, n}\}_{n=1}^N$. These standard errors provide a measure of the uncertainty in the learning process across different trials. Throughout all simulations for this example, we set $N = 100$ and $H=40$ as number of estimation steps.

Because the transitions and rewards are fully known, we can compute exact value functions using Bellman's equations for each stationary policy; see \citet{Agarwal2019ReinforcementLT} for a thorough discussion. We compare: the performance of individual ``expert'' policies (denoted $\pi_1, \pi_2, \ldots$); the best mixture policy in the class $\gC(\Pi)$, denoted $q^\star \Pi$ ; the optimal stationary policy over the entire policy space.

A description of the parameters used in the experiments is provided in Appendix \ref{sec:sim:setting1}.

\subsubsection{Direct Tabular Learning and Comparison with Baselines}\label{sec:comp_baselines}

We consider a particular case of orchestration, namely the one where each expert corresponds to a particular action (edge). In this scenario, there are no real experts, as the policies are directly associated with individual actions.

When $K = |\gA|$ and the policies $\Delta = (\pi_a)_{a \in \gA}$ are given
by Dirac masses, i.e., $\pi_a(s) = \delta_a$
for all $s \in \gS$, then
\[
\gC(\Delta) = \bigl\{ p\Delta, \ p \in \gP(\gA)^{\gS} \bigr\}
\]
is the set of all stationary policies, stated in their tabular form via a direct parametrization (following
the terminology of~\citealp[Section~3]{Agar21}).

We compare our strategies against traditional reinforcement learning baselines:
\begin{itemize}
    \item \textbf{Q-learning vs Tabular Policy Learning with Tabular Advantage Learning (Figure \ref{fig:comparison_baselines} (a)):} We compare the value obtained using the policy learned with our strategy to the QL policy. By QL policy, we refer to the greedy policy derived from the Q-values produced by Algorithm \ref{alg:QL} (see Appendix \ref{sec:alg:std}), where actions are selected by maximizing the Q-value for each state.

    For this comparison, we ignore the cost of policy updates (which involves a single matrix multiplication) and focus solely on the number of TD updates required to estimate the advantage function at each step. Importantly, one TD update in our method has the same computational cost as one QL update. Therefore, it is meaningful to compare the two approaches directly in terms of the number of such updates. The x-axis reports the total number of TD updates used by our algorithm to estimate the advantage functions ($t \cdot H$), which corresponds to the total number of QL updates performed by the baseline.

As shown in Figure \ref{fig:comparison_baselines}, the orchestration strategy outperforms QL as it converges much faster to the optimal value overall.

\item \textbf{Double-DQN vs Tabular Policy Learning with Neural Network Advantage Learning (Figure \ref{fig:comparison_baselines} (b)) :} 
We compare the value obtained using tabular policy learning with the potential function combined with NN Advantage estimation to the Double DQN (D-DQN) policy. By Double DQN policy, we refer to the greedy policy derived from the Q-values produced by Algorithm \ref{alg:doubleDQN}, where actions are selected by maximizing the Q-value for each state. As before, for the comparison, we neglect the update on the policy and focus on the number of NN updates needed to approximate the advantage function at each step. One NN update is computationally equivalent to one step of the Double DQN algorithm.

\end{itemize}

\begin{figure}[h]
    \centering
    \begin{subfigure}[h]{0.49\textwidth}
        \centering
\includegraphics[width=\textwidth]{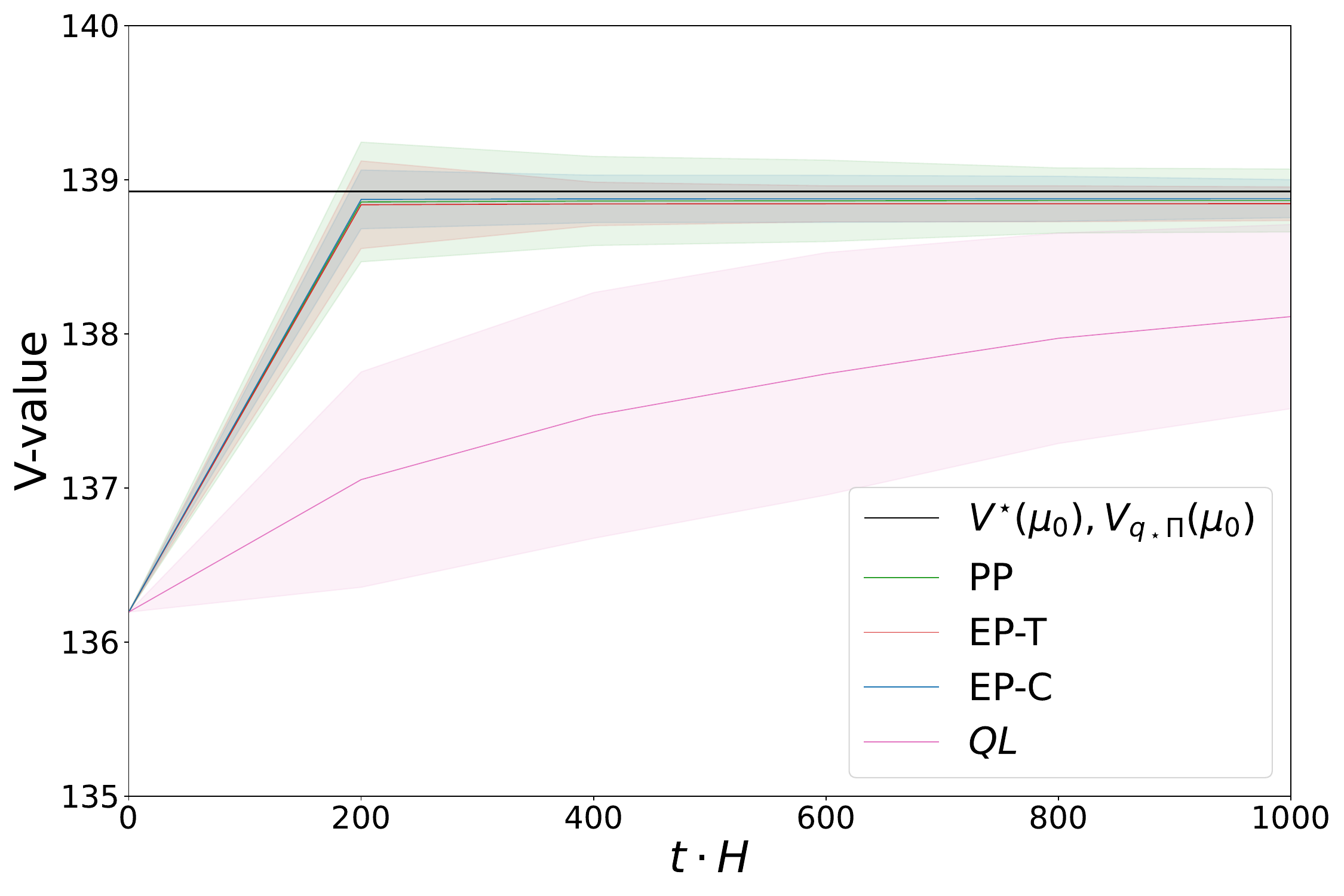}
        \caption{Tabular Advantage Learning vs QL.}
    \end{subfigure}
    \begin{subfigure}[h]{0.49\textwidth}
        \centering
\includegraphics[width=\textwidth]{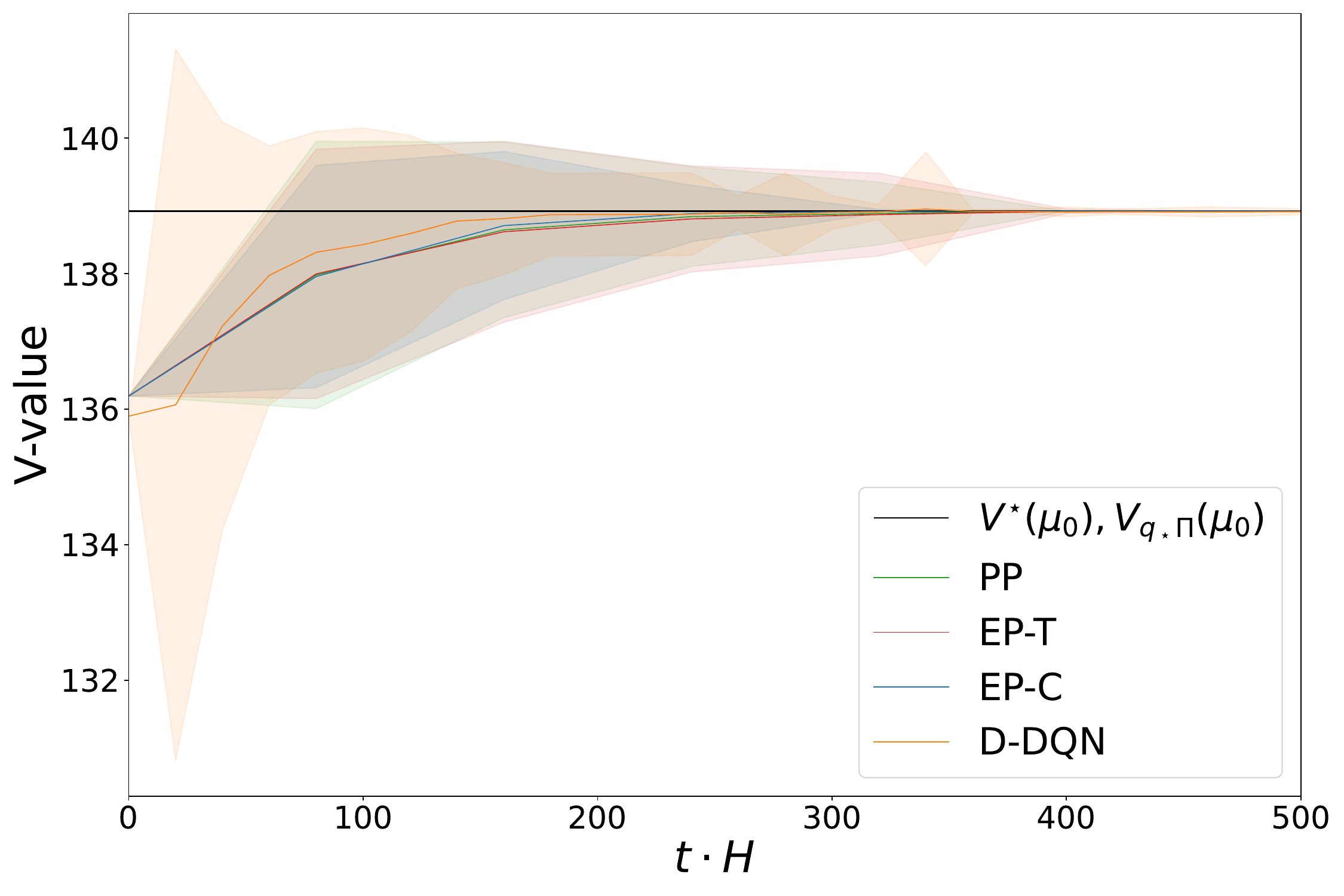}
        \caption{NN Advantage Learning vs Double-DQN.}
    \end{subfigure}
    \caption{Comparison of direct tabular learning performance against baseline methods.}
    \label{fig:comparison_baselines}
\end{figure}

\subsubsection{Learning Dynamics with Expert Policies}

In this section we consider the expert policies introduced in Section~\ref{sec:stoch_model}. In particular, for this experiment, we limit the scope to policies $\pi_1, \pi_2, \pi_4$.

Figure~\ref{fig:learning_curves} illustrates the evolution of the average value achieved by the learned mixture policies (compute as in \eqref{eq:avg_value}), where the x-axis represents policy updates. 

The constant lines in the figure indicate the values of the individual expert policies, the best mixture policy $V_{q^\star \Pi}(\mu_0)$, and the optimal policy $V^\star(\mu_0)$, which serves as a reference for comparison.

\begin{figure}[h]
    \centering
    \begin{subfigure}[t]{0.49\textwidth}
        \centering
        \includegraphics[width=\linewidth]{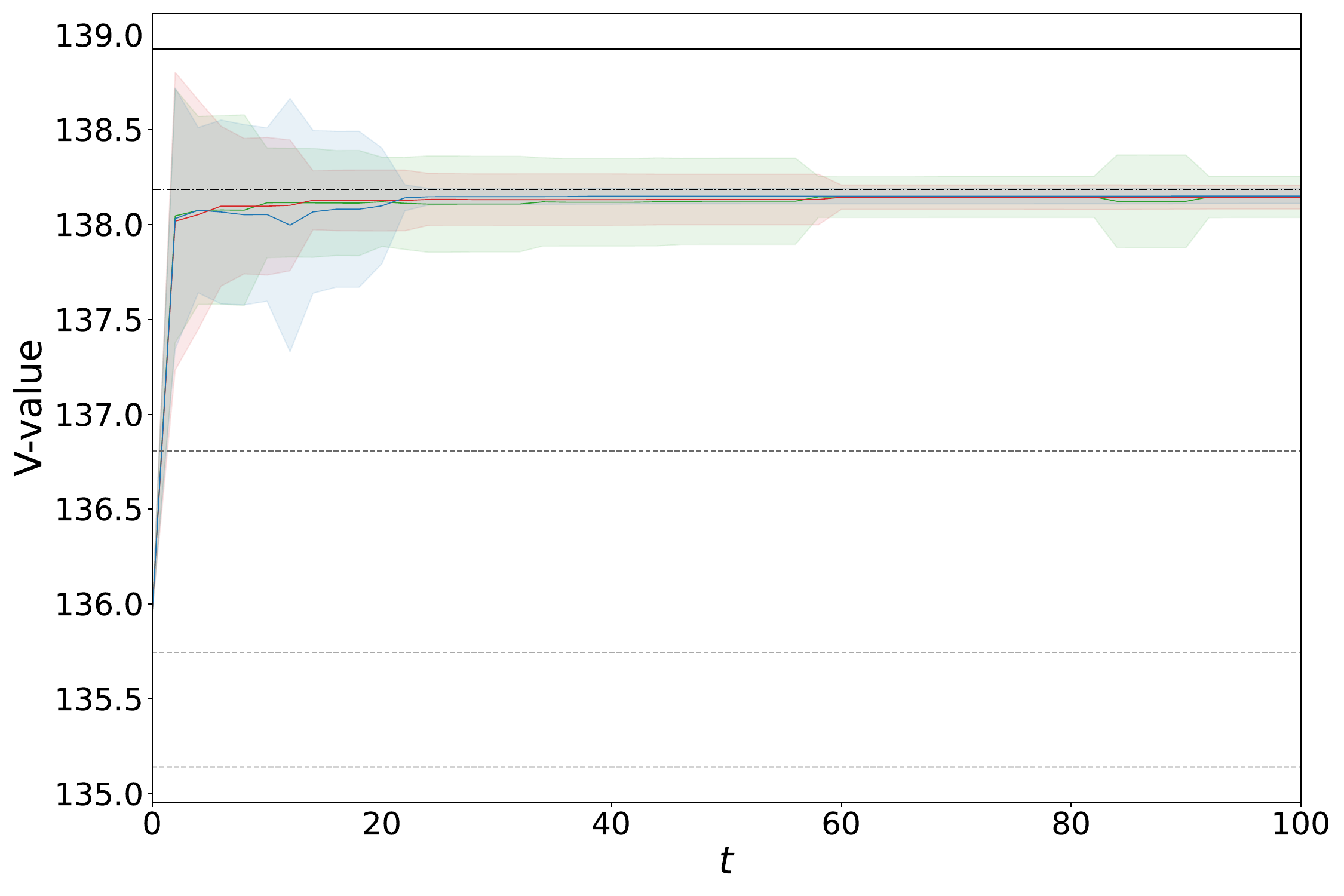}
        \caption{Tabular policy learning with tabular advantage learning.}
    \end{subfigure}
    \hfill
    \begin{subfigure}[t]{0.49\textwidth}
        \centering
       \includegraphics[width=\linewidth]{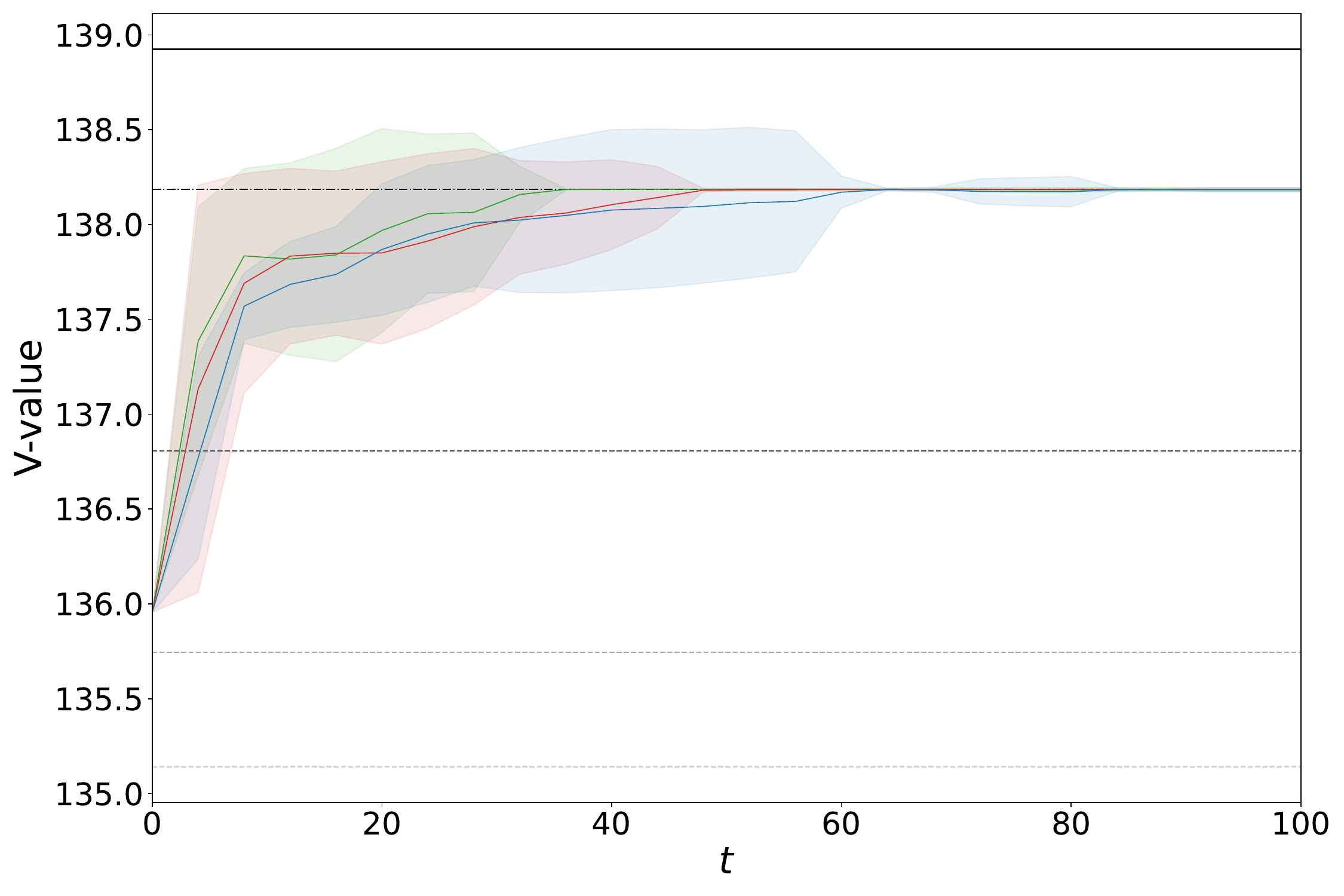}
        \caption{Tabular policy learning with NNAL.}
        
    \end{subfigure}

    \vspace{0.4cm}

    \begin{subfigure}[t]{0.49\textwidth}
        \centering
        \includegraphics[width=\linewidth]{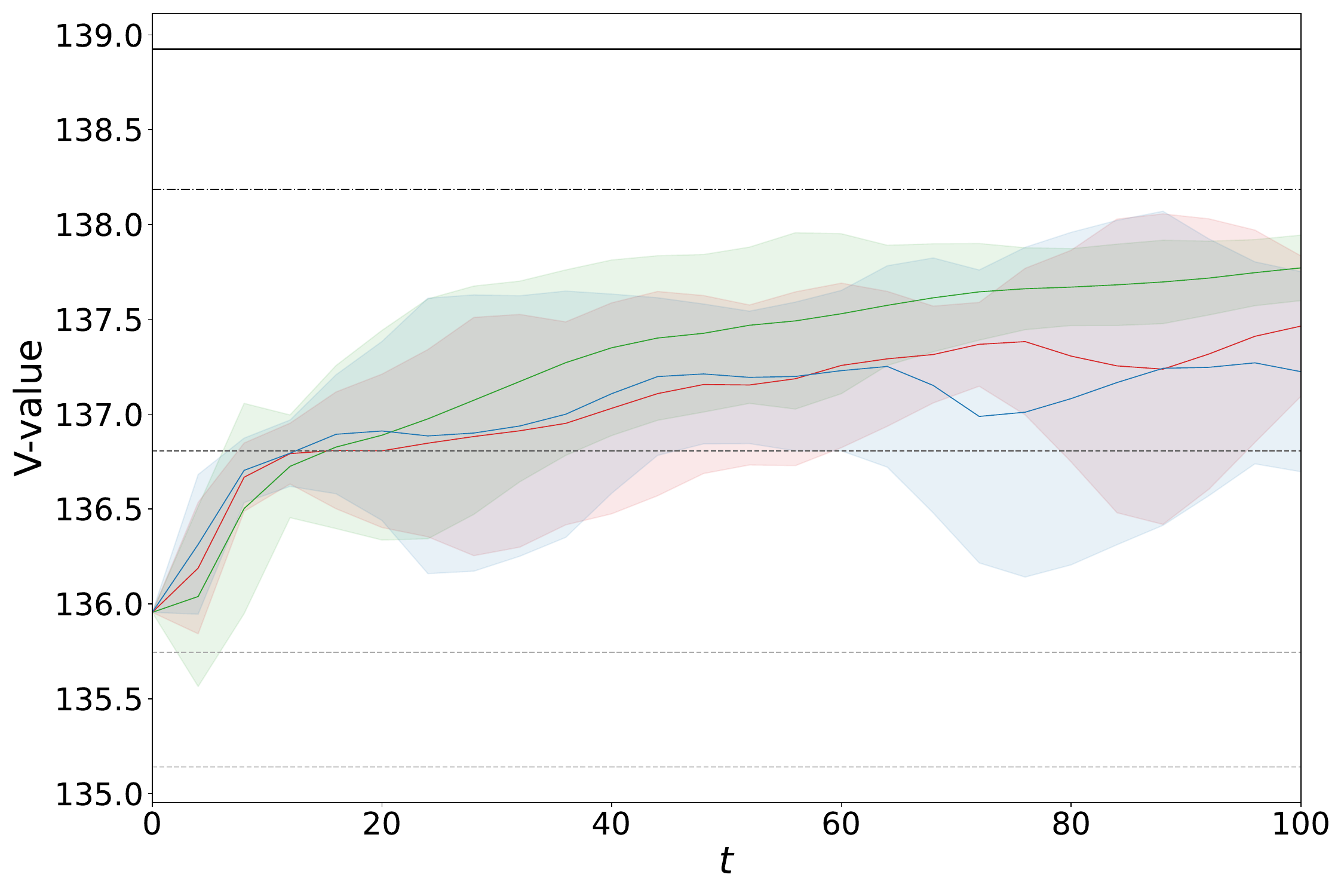}
        \caption{NN policy learning with NNAL.}
    \end{subfigure}
    \hfill
    \begin{subfigure}[t]{0.49\textwidth}
        \centering
         \includegraphics[width=0.35\linewidth]{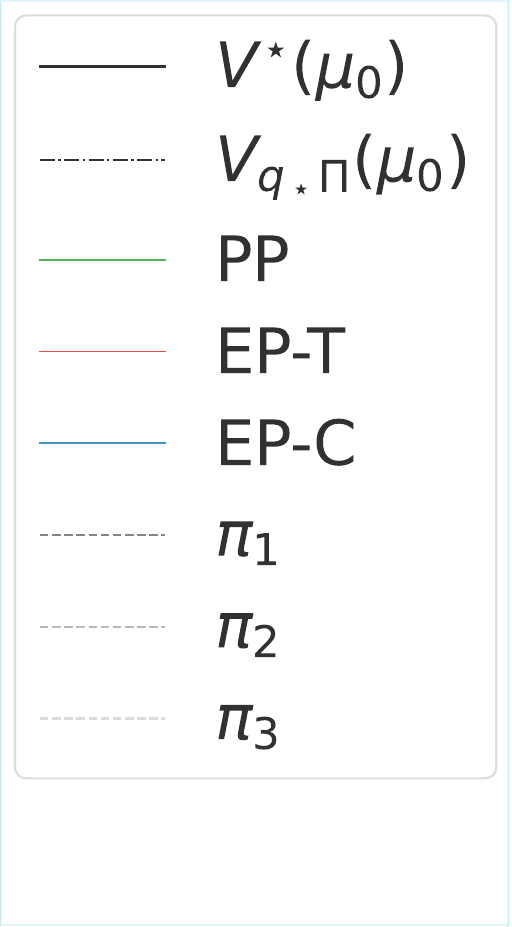}
    \end{subfigure}

    \caption{Convergence of average performance under different orchestration and learning schemes. The legend (bottom right) is shared across all plots.}
    \label{fig:learning_curves}
\end{figure}

While all three plots in Figure~\ref{fig:learning_curves} illustrate the evolution of the policy value as a function of policy updates, Figures~\ref{fig:learning_curves}(a) and \ref{fig:learning_curves}(b) depict the same type of policy update—namely, a tabular update—using different simulators (TD and NN, respectively). In contrast, Figure~\ref{fig:learning_curves}(c) differs in that it represents neural network (NN)  updates of the policy, utilizing a NN simulator. While the first two plots illustrate a single-step update, which can be computed with a simple matrix multiplication, the third one involves significantly more complex computations. In fact, training a general neural network requires performing multiple forward and backward passes, with a computational complexity of approximately \( O( l \cdot d^2) \), where \( l \) is the number of layers, and \( d \) is the dimensionality of each layer.

\paragraph{State Space Dominance Analysis} Table \ref{tab:state-space} summarizes the proportion of the state space dominated by each expert policy (i.e., the frequencies with which each expert policy appears in the mixture $q^\star$) for all the strategies of case (a) of Figure\ref{fig:learning_curves}. These frequencies reveal the contribution of each expert policy to the improved overall performance. 

\begin{table}[t]
\centering
\caption{Proportion of the state space controlled by each expert policy under different strategies in the tabular policy learning with advantage learning framework.}
\label{tab:state-space}
\begin{tabular}{l c c c}
\toprule
\textbf{Strategy} & $\pi_1$ & $\pi_2$ & $\pi_4$\\
\midrule
Exponential fixed $\eta$ &  0.36 & 0.35 & 0.29\\
Exponential  $\eta_t$    & 0.23 &  0.41& 0.36\\
Polynomial    & 0.32 & 0.30 &0.38\\
\bottomrule
\end{tabular}
\end{table}

For the simulation, we simplify the state space. At each time step, an event occurs that modifies the queues, as described in Appendix \ref{appendix:matching}, particularly in \eqref{eq:transition_matrix}. The queue lengths at time $t$ contain all the necessary information for determining possible matches. Consequently, our algorithm represents the state solely by the number of items in each queue at each time step. Therefore, for each model, the state space has size $L^{\text{\#nodes}}$, in this case $L^4$.

\subsection{Organ Exchange Model}\label{sec:organ_exchange}

Here, we apply our orchestration strategies to the stochastic matching framework of \citealp{jonckheere2023generalized}, which is motivated by organ transplantation settings and captures key structural features of such networks while remaining tractable for simulation. The framework is designed to test scalability, adaptivity, and robustness in a heterogeneous, high-stakes domain. We evaluate how effectively the proposed methods navigate the complexity of blood-type compatibility, urgency levels, and asymmetric rewards.

 The compatibility graph (Figure~\ref{fig:org_don_graph}) is bipartite: one set represents donors, and the other represents recipients. Each donor or recipient belongs to one of four blood-type groups $\{0, A, B, AB\}$. Within each group, nodes are subdivided by urgency level $u \in \{0,1,2\}$ (low, medium, high).

Each node $i$ is represented by a tuple $(G_i, u_i)$, where $G_i \in \{0, A, B, AB\}$ denotes the blood type and $u_i \in \{0, 1, 2\}$ represents the urgency level. We use the same notation for rewards, transition kernels, and expert policies as in Section~\ref{sec:stoch_model}, with the following modifications:

\begin{itemize}
    \item For each node $i = (G_i, u_i)$ where $u_i \in \{0, 1\}$, the relocation "next node" is given by $N_i = (G_i, u_i + 1)$, which corresponds to the node with the same blood type but a higher urgency level. The relocation probability is denoted by $\nu_i > 0$.
    \item For nodes at the highest urgency level, $(G_i, 2)$, we set $\nu_i = 0$. This means that any departure from these nodes is final, as there is no higher urgency level.
    \item The reward function distinguishes between \emph{final} departures (from the highest urgency level) and \emph{non-final} departures, with the latter being penalized less. Specifically, for a given state $s = (\varrho, (G_i, u_i), e)$ and an action $a$, the reward function is:
    \[
    r(s,a) = - D_{u_i} \cdot 1_{e=\text{departure}} - R_{u_i} \cdot 1_{e=\text{departure}} + g_{i,a},
    \]
    where $R_{u_i} < D_{u_i}$. Here, $D_{u_i}$ and $R_{u_i}$ are constants associated with the cost of an item either departing the system or being relocated from node $i$. In this model, these constants depend only on the urgency level $u_i$, not the blood type $G_i$.
\end{itemize}

This model illustrates how different urgency levels, together with complex compatibility constraints, can be effectively managed within the same Markov decision process framework. As in previous sections, combining multiple expert policies using potential-based learning methods facilitates efficient exploration of decisions across both blood types and urgency levels.

Model specifications are provided in tables \ref{tab:node_params}, \ref{tab:urgency_params}, and \ref{tab:global_params} in Appendix \ref{sec:sim:setting2}.

\subsubsection{Orchestration, Direct Learning, and Baselines}

For this experiment, we use the expert policies $\pi_1, \pi_2, \pi_3, \pi_4$ introduced in Section~\ref{sec:stoch_model}, and focus on Learning Scheme 3, which uses neural network policy learning (with NNAL). This choice is motivated by the high-dimensional nature of the organ exchange environment, where simpler methods are computationally infeasible.

\paragraph{Performance Criterion}  As in the previous case, we compare the performance of the algorithms in terms of $V_{{q}_{t} \Pi}(\mu_0)$.

We follow the same evaluation protocol as in the previous experiment: for each time step \( t \), we perform \( N \) independent runs of the learning process, and report the average value$\frac{1}{N} \sum_{n=1}^N V_{{q}_{t, n} \Pi}(\mu_0)$ over these runs. 
Shaded areas indicate \(\pm 2\) standard errors, computed from the variability across the \( N \) learned policies at each time step.

In this setting, the transition and reward structure remains fully known. However, unlike in the previous experiment, the state space is too large to allow for exact computation or storage of value functions. As a result, we rely on Monte Carlo techniques to estimate policy values: specifically, we report the average cumulative reward obtained over 200 steps of interaction, averaged across 5{,}000 independent runs.

Furthermore, due to the size and complexity of the environment, we can no longer include the true optimal values for comparison. As we will later observe, traditional methods such as TD learning or DQN converge very slowly in this setting, making them impractical as baselines.

\paragraph{Performance Comparison}
We compare NN policy learning via Neural Network Advantage Learning (NNAL) in both expert-guided and direct action settings, and benchmark the resulting policies against a Double DQN baseline. Consistent with prior evaluations, we focus on the number of neural network updates required to approximate the advantage function at each step. Since all learning approaches rely on networks of comparable size, we treat each policy update in NNAL as equivalent to one Double DQN step.

\paragraph{Figure~\ref{fig:donor_learning_curves}: Identifying the Best Among Many}  
This figure illustrates the effectiveness of orchestration when learning across a diverse set of experts. With three expert policies and a total of 27 possible actions, the orchestrated learner rapidly identifies and converges to the performance of the best single experts. In contrast, the direct learning approaches initially show fast gains due to broad exploration but soon plateau, unable to refine their behavior in such a large, complex action space.  
The orchestrated approach, by efficiently leveraging expert priors, narrows the search space early and focuses learning where it matters most. This result powerfully demonstrates that, in environments with many possible actions, orchestration is not merely beneficial; it yields a marked improvement over direct learning methods.

\begin{figure}[H]
    \centering
    \includegraphics[width=0.6\textwidth]{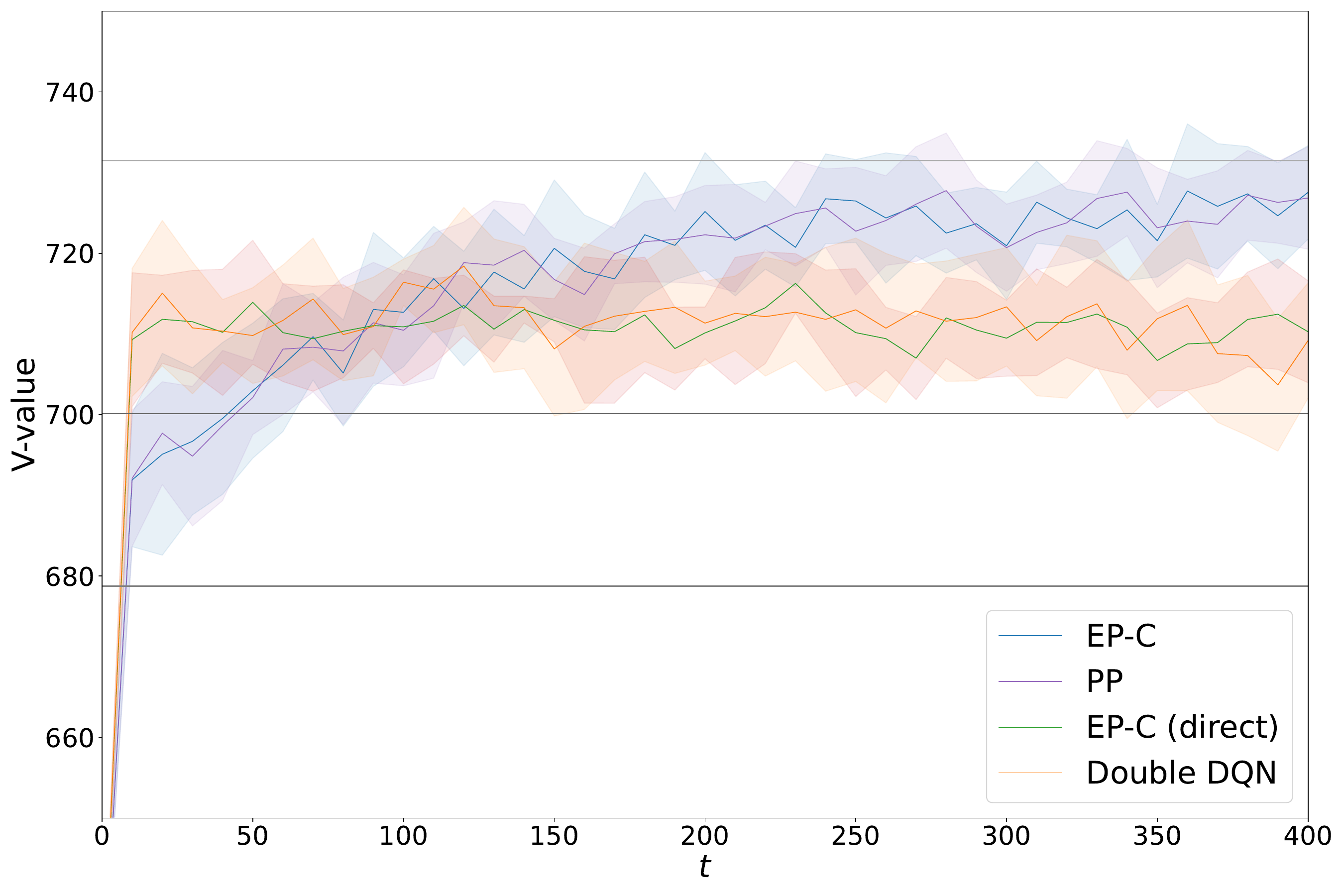}
    \caption{Learning curves for the donor exchange experiment.}
    \label{fig:donor_learning_curves}
\end{figure}

\paragraph{Figure~\ref{fig:donor_learning_curves_1}: Improving Beyond the Best Expert}  
In this experiment, we restrict the expert set to just two policies: $\pi_1$ (match the longest) and $\pi_2$ (greedy max-payoff). Remarkably, even with such a small set, orchestration does not merely select the better expert: it learns an adaptive combination that improves upon both.  
This result is especially important in life-and-death domains like organ exchange, where even small improvements can have critical long-term consequences. Here, orchestration provides an adaptive mechanism that outperforms static expert policies, delivering gains that would be otherwise difficult to achieve (or would require very long training) through direct learning or unreachable fixed expert selection alone.
This highlights orchestration’s potential as a powerful decision-support tool.

\begin{figure}[H]
    \centering
    \includegraphics[width=0.6\textwidth]{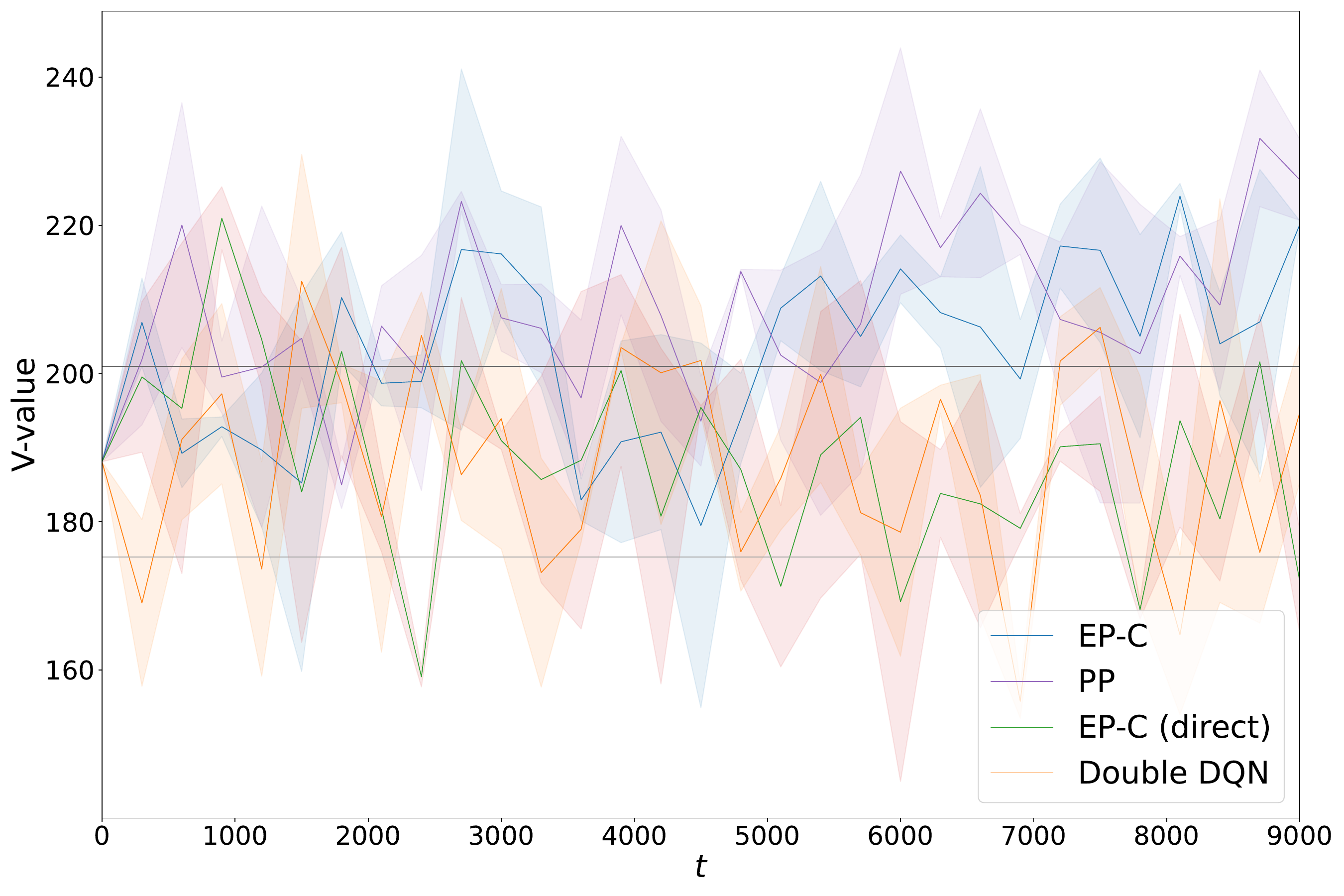}
    \caption{Learning curves for the donor exchange experiment.}
    \label{fig:donor_learning_curves_1}
\end{figure}

\paragraph{Summary}  
Together, these results make a compelling case for the use of orchestration in high-dimensional, high-stakes environments. By building on structured prior knowledge and adapting over time, orchestration delivers faster convergence, better performance, and more reliable policies than either direct learning or individual expert strategies. See Appendix~\ref{sec:sim:setting2} for full reproducibility details.

\section{Conclusion and Future Research}
We have presented a general framework for orchestrating expert policies in reinforcement learning, with a focus on stochastic matching problems. Our approach extends prior orchestration methods by accommodating biased estimators, leveraging potential-based strategies, and scaling to high-dimensional settings via actor-critic architectures. We introduce a novel finite-time bias bound for TD learning, which enables the use of learned advantage estimates while preserving theoretical guarantees. Our empirical evaluation demonstrates that the proposed methods consistently outperform both individual expert policies and standard RL baselines, achieving superior performance and faster convergence in both synthetic and realistic settings.

As a promising direction for future research, we propose extending the orchestration framework to non-stationary environments, where the underlying Markov Decision Process (MDP) may evolve across latent contexts. This requires both the detection of context shifts, potentially via monitoring value function dynamics, and the maintenance of adaptive distributions over possible MDP regimes. Pursuing this direction involves formalizing and evaluating orchestration in dynamically changing environments, where the learner must adapt not only to inherent stochasticity but also to structural non-stationarity. Such an extension would broaden the applicability of our framework to real-world domains, such as financial markets, logistics, or online marketplaces, where adaptability, robustness, and interpretability are all essential.

\bibliographystyle{elsarticle-harv} 
\bibliography{bibliography}

\newpage
\vfill

\begin{center}
{\Large \bf Supplementary material
for \bigskip \\
``{Online Matching via Reinforcement Learning:
An Expert Policy Orchestration Strategy}'' 

}
\end{center}

\begin{itemize}
\item Appendix~\ref{sec:thmProofs}: contains the proofs of Theorems \ref{th:cesaro} and \ref{th:HP-estim}
\item Appendix~\ref{sec:proof_lemma}: contains the proof of Lemma \ref{lemma:adv-bias-bound}.
\item Appendix~\ref{app:pg-comparison}: provides a detailed comparison between our neural policy learning scheme and classical neural policy gradient methods.

\item Appendix~\ref{sec:alg}: provides the pseudocode for the algorithms presented in the paper.
\item Appendix~\ref{appendix:matching}: contains the formal description of the MPD and of the expert policies.
\item Appendix~\ref{sec:sim}: details the simulation settings and includes a parameter study for both scenarios.
\end{itemize}

\appendix

\section{Proofs of Theorems \ref{th:cesaro} and \ref{th:HP-estim}} \label{sec:thmProofs}

\paragraph{Preparation} 
for a given stationary policy $\pi$, we introduce, for each $s \in \gS$,
\[
\mu^{(s_0,\pi)}(s) = (1-\gamma) \sum_{t=0}^{+\infty} \gamma^t \, \mathbb{P}^{(s_0,\pi)}(s_t = s)\,,
\]
i.e., $\mu^{(s_0,\pi)}$ is the discounted state visitation distribution
starting from $s_0$ and taking actions drawn by $\pi$.
Moreover, from \eqref{eq:adv_regret} it follows that
\\
\begin{align}
\label{eq:adv-guar-tilde}
\max_{k \in [K]} \sum_{t=1}^T \widetilde{A}_{q_t\Pi}(s,k)
- \sum_{t=1}^T \smash{\overbrace{\sum_{j \in [K]} q_t(j|s) \widetilde{A}_{q_t\Pi}(s,j)}^{= 0}} 
= \max_{k \in [K]} \sum_{t=1}^T \widetilde{A}_{q_t\Pi}(s,k)
\leq \frac{1}{(1-\gamma)}\,B_{T,K}\,.
\end{align}
Finally, consider the following lemma.
\begin{lemma}[performance difference lemma; see~{\citealp[Lemma~6.1]{KL02}}]\label{lemma:perf_diff}
For any pair $\pi,\pi'$ of stationary policies and all
states $s_0$,
\begin{align*}
V_{\pi}(s_0) - V_{\pi'}(s_0) 
= \frac{1}{1-\gamma}
\sum_{s \in \gS} \mu^{(s_0,\pi)}(s) \sum_{a \in \gA} \pi(a|s) \, A_{\pi'}(s,a)\,.
\end{align*}
\end{lemma}

\subsection{Proof of Theorem \ref{th:cesaro} (analysis in expectation)}\label{sec:proofthm1}
\begin{proof}[Proof of Theorem \ref{th:cesaro}]
The first part is a restatement of Lemma \ref{lemma:perf_diff}: for any pair $q,q' \in \gP\bigl([K]\bigr)^{\gS}$
of state-dependent weights and all states $s_0$:
\begin{align*}
V_{q\Pi}(s_0) - V_{q'\Pi}(s_0)  &=\frac{1}{1-\gamma}
\sum_{s \in \gS} \mu^{(s_0,q\Pi)}(s) \sum_{a \in \gA} q\Pi(a|s) \, A_{q'\Pi}(s,a) \\
& = \frac{1}{1-\gamma}
\sum_{s \in \gS} \mu^{(s_0,q\Pi)}(s) \sum_{k \in [K]} q(k|s) \, A_{q'\Pi}(s,k)\,,
\end{align*}
It follows that
\begin{align}
\label{eq:V-cesaro-diff-tq}
& V_{q^\star\Pi}(s_0) - \frac{1}{T} \sum_{t=1}^T V_{q_t\Pi}(s_0)
=  \frac{1}{(1-\gamma)T}
\sum_{s \in \gS} \mu^{(s_0,q^\star\Pi)}(s) \sum_{k \in [K]} q^\star(k|s) \sum_{t=1}^T A_{q_t\Pi}(s,k)\,.
\end{align}

Then, by the tower rule, Assumption~\ref{lm:est-eps} implies that:
\begin{equation}
\label{eq:tower-rule-loose}
- \epsilon \leq {\mathbb{E}} \bigl[ \widetilde{A}_{q_t\Pi}(s,k) \bigr] - A_{q_t\Pi}(s,k)\leq \epsilon\,.
\end{equation}
Thus, taking expectations in~\eqref{eq:V-cesaro-diff-tq} gives:
\begin{align}
\label{eq:V-cesaro-diff-tq-eps}
& V_{q^\star\Pi}(s_0) - \frac{1}{T} \sum_{t=1}^T {\mathbb{E}} \bigl[ V_{q_t\Pi}(s_0) \bigr] \leq \frac{\epsilon}{1-\gamma} + \\[0.15cm]
\nonumber
& \widetilde{\mathbb{E}} \left[ \frac{1}{(1-\gamma)T}
\sum_{s \in \gS} \mu^{(s_0,q^\star\Pi)}(s) \sum_{k \in [K]} q^\star(k|s) \sum_{t=1}^T \widetilde{A}_{q_t\Pi}(s,k) \right]
\end{align}
From~\eqref{eq:adv-guar-tilde}, we know that, almost surely, for all $s$:
\begin{align}\label{eq:tA-BTK-regr}
\sum_{k \in [K]} q^\star(k|s) \sum_{t=1}^T \widetilde{A}_{q_t\Pi}(s,k)
& \leq \max_{k \in [K]} \sum_{t=1}^T \widetilde{A}_{q_t\Pi}(s,k) \leq \frac{1}{(1-\gamma)}B_{T,K}.
\end{align}
Substituting this inequality into~\eqref{eq:V-cesaro-diff-tq-eps} gives:
\[
V_{q^\star\Pi}(s_0) - \frac{1}{T} \sum_{t=1}^T {\mathbb{E}} \bigl[ V_{q_t\Pi}(s_0) \bigr]
\leq \frac{\epsilon}{1-\gamma} + \frac{B_{T,K}}{(1-\gamma)^2 T}\,,
\]
as desired.
\end{proof}

\subsection{Proof of Theorem \ref{th:HP-estim} (analysis in high probability)}\label{sec:proofthm2}
\begin{proof}[proof of Theorem \ref{th:HP-estim}]
    We use again the (in)equalities~\eqref{eq:V-cesaro-diff-tq} and~\eqref{eq:tA-BTK-regr},
which hold with probability~1,
and based on Assumption~\ref{lm:est-eps}, we only need to explain why,
with probability at least $1-\delta$,
\begin{multline*}
\sum_{s \in \gS} \mu^{(s_0,q^\star\Pi)}(s) \sum_{k \in [K]} q^\star(k|s) \sum_{t=1}^T \widetilde{\mathbb{E}}\bigl[ \widetilde{A}_{q_t\Pi}(s,k) \,\big|\, \gF_{t-1} \bigr] 
\leq \sum_{s \in \gS} \mu^{(s_0,q^\star\Pi)}(s) \sum_{k \in [K]} q^\star(k|s) \sum_{t=1}^T \widetilde{A}_{q_t\Pi}(s,k)
+ \frac{1}{(1-\gamma)} \sqrt{2 T \ln \frac{1}{\delta}}\,.
\end{multline*}
The inequality above indeed follows from the Hoeffding-Azuma lemma, applied to the martingale difference sequence
\[
X_t = \sum_{s \in \gS} \mu^{(s_0,q^\star\Pi)}(s) \sum_{k \in [K]} q^\star(k|s) \, \Bigl(
\widetilde{A}_{q_t\Pi}(s,k) - {\mathbb{E}} \bigl[ \widetilde{A}_{q_t\Pi}(s,k) \,\big|\, \gF_{t-1} \bigr] \Bigr)\,,
\]
whose increments are bounded by $2/\bigl((1-\gamma)\bigr)$.
\end{proof}

\section{Proof of Lemma \ref{lemma:adv-bias-bound}: Uniform Bias Contraction for $Q$-Function Estimates Under a Stationary Policy}\label{sec:proof_lemma}

\begin{proof}[Proof of Lemma \ref{lemma:adv-bias-bound}]

We first show that \( \widetilde{Q}_{\pi} \) is bounded. We prove by induction that \( 0 \le \widetilde{Q}_{\pi, \tau}(s, a) \le M \) for all \( \tau, s, a \), where \( M := \frac{1}{1 - \gamma} \).

 \( \widetilde{Q}_{\pi, 0}(s,a) = 0 \) satisfies the bound. Assume \( 0 \le \widetilde{Q}_{\pi, \tau}(s,a) \le M \) for all \( s,a \). By the update rule (\eqref{eq:update_rule}), for $s\ne s_\tau, a\ne a_\tau$
\[\widetilde{Q}_{\pi, \tau+1}(s, a) = \widetilde{Q}_{\pi, \tau}(s, a)  \le M\]
 while at $(s_\tau, a_\tau)$ 
\[
\widetilde{Q}_{\pi, \tau+1}(s_\tau, a_\tau) \le (1-\alpha) M + \alpha (1 + \gamma M).
\]
Since \( 1 + \gamma M = \frac{1 - \gamma + \gamma}{1 - \gamma} = M \), we get
\[
\widetilde{Q}_{\pi, \tau+1}(s_\tau, a_\tau) \le M.
\]
Non-negativity is preserved as the update is a convex combination of non-negative terms. 
Hence, the bound holds for all \( \tau \), and \( \widetilde{Q}_{\pi}(s, a)\) is bounded by \( M \) for all $s, a\in \gS, \gA$.

By construction (Remark~\ref{remark:A_contruction}), \( \widetilde{V}_\pi \le M \), so
\[
- M \le \widetilde{A}_\pi(s,a) \le M, \quad \forall s,a.
\]
Finally, using the definition of the estimated advantage function
\begin{align*}
\sum_{a \in \gA} \pi(a|s) \widetilde{A}_\pi(s,a) = \sum_{a \in \gA}\widetilde{Q}_\pi(s,a) - \widetilde{V}_\pi(s) = 0.
\end{align*}

We now define the bias at time $\tau$ for each fixed state-action pair $(s, a) \in \gS \times \gA$ as
\begin{align*}
    b_\tau(s, a) := \widetilde{Q}_{\pi, \tau}(s, a) - Q_{\pi}(s, a),
\end{align*}

The TD update at step $\tau$ is given by:
\begin{align*}
    \widetilde{Q}_{\pi, \tau+1}(s, a) = \begin{cases}
        (1 - \alpha) \widetilde{Q}_{\pi, \tau}(s, a) + \alpha \left( r_\tau + \gamma \widetilde{Q}_{\pi, \tau}(s_{\tau+1}, a'_{\tau+1}) \right) & \text{if } (s, a)= (s_\tau, a_\tau)\\
        \widetilde{Q}_{\pi, \tau}(s, a) & \text{otherwise}
    \end{cases}
\end{align*}

Let \( \mathcal{F}_\tau \) denote the \( \sigma \)-algebra generated by the randomness inherent in the estimation process up to round \( \tau \). Then
\begin{align*}
    &\mathbb{E}\left[b_{\tau+1}(s, a) \;| \;\gF_\tau \right] =\\
    &\quad= \begin{cases}
        \mathbb{E}\left[(1 - \alpha) \left(\widetilde{Q}_{\pi, \tau}(s, a) - Q_{\pi}(s, a)\right) + \alpha \left( r_\tau + \gamma \widetilde{Q}_{\pi, \tau} (s_{\tau+1}, a'_{\tau+1}) - Q_{\pi}(s, a)\right)\;| \;\gF_\tau \right] & \text{if } (s, a)= (s_\tau, a_\tau) \\
        \mathbb{E}\left[\widetilde{Q}_{\pi, \tau}(s, a) - Q_{\pi}(s, a) | \gF_\tau \right]& \text{otherwise}\end{cases} \\
        &\quad=\begin{cases}
        (1 - \alpha) b_\tau(s, a) + \alpha \mathbb{E}\left[ r_\tau + \gamma \widetilde{Q}_{\pi, \tau} (s_{\tau+1}, a'_{\tau+1}) - Q_{\pi}(s, a)\mid \gF_\tau\right] & \text{if }(s, a)= (s_\tau, a_\tau)\\
        b_\tau(s, a)& \text{otherwise}\\
    \end{cases}  
\end{align*}
From the Bellman equation for \( Q_\pi \), i.e.,
\[
Q_{\pi}(s_\tau, a_\tau) = \mathbb{E}\left[r_\tau + \gamma Q_{\pi}(s_{\tau+1}, a'_{\tau+1}) \mid  s_\tau, a_\tau\right],
\]
we get
\begin{align*}
    &\mathbb{E}\left[b_{\tau+1}(s, a) | \gF_\tau \right] =\\
    &\quad=\begin{cases}
        (1 - \alpha) b_\tau(s, a) + \alpha\gamma \mathbb{E}\left[ \widetilde{Q}_{\pi, \tau} (s_{\tau+1}, a'_{\tau+1}) - Q_{\pi}(s_{\tau + 1}, a'_{\tau+ 1})\mid \gF_\tau\right] & \text{if } (s, a)= (s_\tau, a_\tau)\\
        b_\tau(s, a)& \text{otherwise}
    \end{cases} 
\end{align*}
where 
\begin{align*}
    \mathbb{E}\left[ \widetilde{Q}_{\pi, \tau} (s_{\tau+1}, a'_{\tau+1}) - Q_{\pi}(s_{\tau + 1}, a'_{\tau+ 1})\mid \gF_\tau\right] = \sum_{s'} \mathcal{T}(s' \mid s_\tau, a_\tau) \sum_{a'} \pi(a' \mid s') b_\tau(s',a')  
\end{align*}

Conditioning on \( \mathcal{F}_\tau \) corresponds here to taking the expectation over the next state-action pair \( (s_{\tau+1}, a'_{\tau+1}) \), where \( s_{\tau+1} \sim \mathcal{T}(\cdot \mid s_\tau, a_\tau) \) and \( a'_{\tau+1} \sim \pi(\cdot \mid s_{\tau+1}) \).

Let $p_{\pi,\tau}(s,a) = \mathbb{P}((s_\tau, a_\tau) = (s, a))$ be the distribution of visited state-action pairs under the sampling policy at step $\tau$. Taking total expectation over \( (s_\tau, a_\tau) \), we obtain:
\begin{align*}
\mathbb{E}[b_{\tau+1}(s,a)] 
&= \mathbb{E}\left[ \mathbf{1}_{\{(s_\tau, a_\tau)=(s,a)\}} \left( (1 - \alpha) b_\tau(s,a) + \alpha \gamma \sum_{s'} \mathcal{T}(s' \mid s,a) \sum_{a'} \pi(a' \mid s') b_\tau(s',a') \right) \right] \\
&\quad + \mathbb{E}\left[ \mathbf{1}_{\{(s_\tau, a_\tau) \ne (s,a)\}} b_\tau(s,a) \right] \\
&= p_{\pi, \tau}(s,a) \left( (1 - \alpha) \mathbb{E}[b_\tau(s,a)] + \alpha \gamma \sum_{s'} \mathcal{T}(s' \mid s,a) \sum_{a'} \pi(a' \mid s') \mathbb{E}[b_\tau(s',a')] \right) \\
&\quad + (1 - p_{\pi, \tau}(s,a)) \mathbb{E}[b_\tau(s,a)]\\
&= p_{\pi, \tau}(s,a) \left( (1 - \alpha) \mathbb{E}[b_\tau(s,a)] + \alpha \gamma  \left(P_\pi\mathbb{E}[b_\tau]\right)(s',a') \right)  + (1 - p_{\pi, \tau}(s,a)) \mathbb{E}[b_\tau(s,a)]
\end{align*}
where $P$ is the transition operator  defined as
\begin{align*}
    (P_\pi b)(s,a) := \sum_{s'} \mathcal{T}(s' \mid s,a) \sum_{a'} \pi(a' \mid s') \mathbb{E}[b(s',a')],
\end{align*}

From the uniform geometric ergodicity theorem \citep[Theorem 16.0.2]{meyn2009markov}, for any bounded function $f : \mathcal{S} \times \mathcal{A} \to \mathbb{R}$, there exist constants $R > 0$ and $\rho \in (0,1)$ such that
\[
\sup_{(s,a)} \left| \mathbb{E}_{s_0=s}[f(s_\tau, a_\tau)] - \mathbb{E}_{(s,a) \sim d_\pi}[f(s,a)] \right| \le 2R \|f\|_\infty \rho^\tau.
\]
Apply this to the indicator function $f_{(s,a)}(s',a') := \mathbf{1}_{(s',a')=(s,a)}$, yielding
\[
|p_{\pi, \tau}(s,a) - d_\pi(s,a)| \le 2R \rho^\tau,
\]
where $d_\pi(s,a) := d_\pi(s) \pi(a \mid s)$ is the stationary distribution over state-action pairs.

Thus, we write
\[
p_{\pi, \tau}(s,a) = d_\pi(s,a) + \Delta_\tau(s,a), \quad \text{with } |\Delta_\tau(s,a)| \le 2R \rho^\tau.
\]

Substitute into the update
\begin{align*}
T(b)(s,a)
&= (1 - \alpha d_\pi(s,a) - \alpha \Delta_\tau(s,a)) \mathbb{E}[b(s,a)] \\
&\quad + \alpha \gamma (d_\pi(s,a) + \Delta_\tau(s,a)) (P_\pi \mathbb{E}[b])(s,a).
\end{align*}

Using $\|P_\pi b\|_\infty \le \|b\|_\infty$, we upper bound
\[
\|T(b_\tau)\|_\infty \le \left(1 - \alpha(1 - \gamma) \min_{(s,a)} d_\pi(s,a) + \alpha(1+\gamma) \cdot 2R \rho^\tau \right) \|b_\tau\|_\infty.
\]

Define the contraction rate under the stationary distribution
\[
\kappa := \alpha(1 - \gamma) \min_{(s,a)} d_\pi(s,a),
\quad \text{and the time-dependent perturbation} \quad
\delta_\tau := \alpha(1+\gamma) \cdot 2R \rho^\tau.
\]
Then we obtain the bound:
\[
\|T(b_\tau)\|_\infty \le (1 - \kappa + \delta_\tau) \|b_\tau\|_\infty.
\]

Since \(\delta_\tau\) decays geometrically in \(\tau\), this shows that the bias contracts exponentially with a vanishing additive perturbation
\[
\|\mathbb{E}[b_{\tau+1}]\|_\infty \le (1 - \kappa + \delta_\tau) \|\mathbb{E}[b_\tau]\|_\infty.
\]
By iterating this inequality, we get
\[
\|\mathbb{E}[b_{\tau}]\|_\infty \le \left( \prod_{k=0}^{\tau-1} (1 - \kappa + \delta_k) \right) \|\mathbb{E}[b_0]\|_\infty.
\]
Since \(\delta_k \le C \rho^k\), with \(C := \alpha(1+\gamma) \cdot 2R\), and \(\sum_k \delta_k < \infty\), the product admits the bound
\[
\|\mathbb{E}[b_\tau]\|_\infty \le (1 - \kappa)^\tau \cdot \exp\left( \frac{1}{1 - \kappa} \sum_{k=0}^{\tau-1} \delta_k \right) \|\mathbb{E}[b_0]\|_\infty.
\]
This proves geometric convergence of the expected bias, uniformly over time.

\end{proof}

\subsection{Related Work and Comparison (Extended Discussion)} \label{appendix:related-work-lemma}

The problem of analyzing the bias of temporal difference (TD) learning has been central in reinforcement learning theory. Classical analyses, such as \citet{tsitsiklis1997analysis}, establish almost sure convergence under linear function approximation and diminishing step sizes, using a stochastic approximation framework. \citet{borkar2000ode} generalized this framework to broader stochastic iterative schemes, providing ODE-based convergence guarantees. However, these works focus on asymptotic behavior and do not provide explicit finite-time or bias bounds.

Finite-time analyses have emerged more recently. \citet{SrikantYing2019} provided finite-time mean-squared error (MSE) bounds under constant step sizes, assuming stationary data. \citet{Dalal2018TD} analyzed TD(0) under both i.i.d. and Markovian noise, deriving concentration bounds and bias-variance decompositions, though without isolating the bias explicitly. \citet{Bhandari2018} extended finite-time convergence results under Markovian sampling by leveraging the mixing time, again focusing on overall convergence rates rather than bias isolation. Asymptotic behavior under constant stepsizes (including steady-state bias) has been analyzed for linear SA/TD by \citet{SrikantYing2019} and off-policy bias was addressed asymptotically by \citet{Yu2015ETD, yu2016weak} using emphatic weightings.

Our result contributes to this line of work by deriving an explicit, finite-time bound on the bias of TD learning under constant step sizes, without requiring stationary or Markovian sampling assumptions. Specifically, we capture the effect of non-stationarity through a perturbation term $\delta_t$, and show how the error decomposes into a contractive component, a step-size-induced bias, and an accumulated perturbation decay. Our proof technique uses a recursive expansion and contraction inequality directly on the TD update, offering a modular approach applicable to related stochastic approximation algorithms.

\section{Comparison with Neural Policy Gradient Methods} \label{app:pg-comparison}

Neural policy gradient (NPG) and actor-critic methods are foundational tools in reinforcement learning, particularly suited for high-dimensional or continuous action spaces. Classical methods such as REINFORCE (\citealp{williams1992reinforce}), Advantage Actor-Critic (A2C), and Natural Policy Gradient (NPG) optimize stochastic policies by ascending the gradient of expected returns, using Monte Carlo or bootstrapped advantage estimates \citep{sutton1999policy, Kak01, schulman2015trust}. More recent approaches, including Proximal Policy Optimization (PPO) \citep{schulman2017proximal} and Soft Actor-Critic (SAC) \citep{haarnoja2018soft}, improve stability via regularization and introduce entropy-based exploration.

 While our approach retains the actor-critic structure, consisting of a value-based critic and a parametric actor, there are several distinguishing features in the way policies are updated and represented:

\paragraph{Learning via Potential-Based Advantage Aggregation}
Rather than updating the policy via a gradient estimate of expected return, we use a potential-based aggregation of temporal advantage estimates. This technique stems from online learning theory \citep{CBL06} and avoids the need to compute explicit policy gradients. In particular, the policy network is trained to align with targets derived from weighted historical advantages, modulated by polynomial or exponential potential functions. This design:
\begin{itemize}
    \item sidesteps high-variance gradient estimators common in policy gradient methods;
    \item enables regret-minimizing updates even when advantage estimates are biased;
    \item provides a stable update target without relying on differentiable reward signals or environments.
\end{itemize}

\paragraph{Training Without Explicit Policy Gradients}
Our actor is trained using a loss function that compares the current policy distribution to an advantage-weighted target using a suitable divergence metric (e.g., KL divergence or cross-entropy). This mirrors the behavior of softmax policy updates or trust-region methods but avoids the complexities of natural gradient computation. The update follows:
\[
\mathcal{L}_{\text{actor}} = \ell\left(\mathcal{M}_{\text{target}}, \mathcal{M}_{\phi_t}\right), \quad \phi_{t+1} = \phi_t - \alpha_t \nabla_{\phi_t} \mathcal{L}_{\text{actor}},
\]
where the target distribution $\mathcal{M}_{\text{target}}$ is derived from past estimated advantages via potential-based weighting. This avoids several issues common in policy optimization, such as premature convergence, poor step size tuning, or sensitivity to noisy rewards.

Our method offers an alternative to classical neural policy gradient techniques:
\begin{itemize}
    \item We introduce a non-gradient-based policy learning scheme that is compatible with TD-based advantage estimation and still admits convergence guarantees;
    \item We decouple the critic and actor updates in a way that enables more stable learning, even under non-stationary sampling;
    \item We design a loss function that aligns actor updates with potential-weighted targets, serving as an alternative to softmax parametrizations or entropy regularization;
    \item We demonstrate that this approach scales well in structured, high-dimensional environments, with empirical gains in both performance and convergence speed over gradient-based baselines such as DQN.
\end{itemize}

While related techniques appear in online learning and imitation learning contexts (e.g., potential-based updates or distribution matching), our application of these tools to neural policy learning in reinforcement learning settings represents a novel contribution, particularly under finite-time, biased-advantage conditions.

\section{Algorithms}\label{sec:alg}

This appendix presents the algorithms used throughout the paper to estimate Q-values and update policies. We divide them into two main categories:

\begin{itemize}
    \item \textbf{Estimation Procedures:} these algorithms are used to estimate value functions, either via tabular updates, neural networks, or more advanced techniques like Double DQN. They serve as the backbone of the critic in actor-critic frameworks.
    \item \textbf{Standard Algorithms:} for completeness, we include well-known algorithms like Q-learning, which form a baseline for comparison and are occasionally used for pretraining or bootstrapping.
\end{itemize}

In what follows, we describe each estimation procedure in detail, highlighting its role in our broader framework.

\subsection{Estimation Procedures}\label{sec:alg:est}

\paragraph{Algorithm~\ref{alg:1}: Temporal Difference (TD) Learning}  
This is a classical tabular approach to estimating Q-values under a fixed policy $\pi$, using a one-step bootstrap. At each time step, the value of the current state-action pair is updated toward a target that incorporates both the immediate reward and the expected return from the next state (sampled under $\pi$). While simple and effective in small discrete settings, TD learning becomes infeasible in large or continuous state spaces.

\begin{algorithm}[h]
\caption{Estimation of Q-values via Temporal Difference}
\begin{algorithmic}[1]\label{alg:1}
\STATE \textbf{Input:} state space $\mathcal{S}$, action space $\mathcal{A}$, trasition kernel $\gT: \mathcal{S}\times \mathcal{A} \rightarrow \mathcal{P}(\mathcal{S})$, reward function $\mathcal{R}: \mathcal{S}\times \mathcal{A} \rightarrow [0, 1]$, stationary policy $\pi$, number of iterations $H$, learning rate and discount factor \(\alpha, \gamma \in (0, 1)\).
\STATE \textbf{Initialize:} state $s_0$, \(\widetilde{Q}_{\pi, 0}(s, a) = 0\) for all \(s \in \mathcal{S}, \, a \in \mathcal{A}\)
\FOR{$\tau = 0 \dots H-1$}
    \STATE Observe $s_\tau$, select and execute action $a_\tau \sim \pi(\cdot | s_\tau)$, receive reward $r_\tau$, and observe next state $s_{\tau+1}$.
    \STATE Select $a'_{\tau+1}\sim \pi(\cdot|s_{\tau+1})$ and update the estimate:
    \begin{align*}
        \widetilde{Q}_{\pi, \tau+1}(s_\tau, a_\tau) &= (1-\alpha)\widetilde{Q}_{\pi, \tau}(s_\tau, a_\tau) + \alpha \left( r_\tau + \gamma \widetilde{Q}_{\pi, \tau}(s_{\tau+1}, a'_{\tau+1}) \right)
        \\
        \widetilde{Q}_{\pi, \tau+1}(s, a) &= \widetilde{Q}_{\pi, \tau}(s, a) \;\;\;\;\;\;\;\;\;\;\;\; \forall s, a \ne s_\tau, a_\tau
    \end{align*}

    \STATE

\ENDFOR
\STATE Set \(\widetilde{Q}_{\pi} = \widetilde{Q}_{\pi, H}\) 
\STATE \textbf{Output:} Q-values estimate \(\widetilde{Q}_{\pi}\) for the policy-update step.
\end{algorithmic}
\end{algorithm}

\paragraph{Algorithm~\ref{alg:2}: Neural Network-Based Q-Estimation}  
To extend Q-estimation to larger domains, we approximate the Q-function using a neural network. This variant of TD learning trains a network $\mathcal{N}_\theta$ to regress on Q-value targets using backpropagation. At each step, the target is constructed using a fixed (detached) copy of the network, preventing instability due to recursive bootstrapping. This formulation underlies many deep RL algorithms and serves as a flexible critic in our framework.

\begin{algorithm}[H]
\caption{Estimation of Q-values via Neural Network Training}
\begin{algorithmic}[1]\label{alg:2}
\STATE \textbf{Input:} state space $\mathcal{S}$, action space $\mathcal{A}$, transition kernel $\mathcal{T}: \mathcal{S}\times \mathcal{A} \mapsto \mathcal{P}(\mathcal{S})$, reward function $\mathcal{R}: \mathcal{S}\times \mathcal{A} \mapsto [0, 1]$, stationary policy $\pi$, a parameters space $\Theta$, a loss function $\mathcal{L}:\Theta \mapsto \mathbb{R}$ (e.g., mean squared error), number of iterations $H$, learning rate decay factor and discount factor \(\lambda, \gamma \in (0, 1)\).
\STATE \textbf{Initialize:} state $s_0$, starting learning rate $\eta_0$, a neural network $\mathcal{N}_{\theta_0}$ with parameters $\theta_0 \in \Theta$, consisting of two fully connected layers with ReLU activations, followed by a linear output layer (Figure \ref{fig:neural_network}).

\FOR{$\tau = 0 \dots H-1$}
    \STATE Observe $s_\tau$, select and execute action $a_\tau \sim \pi(\cdot | s_\tau)$, receive reward $r_\tau$, and observe next state $s_{\tau+1}$.
    \STATE Compute
    \begin{itemize}
        \item Predicted Q-value:  $\mathcal{N}_{\theta_\tau}(s_\tau, a_\tau)$ 
        \item Target action: $a'_{\tau+1}\sim \pi(\cdot | s_{\tau+1})$.
        \item Target Q-value: $\mathcal{N}_{\theta'_\tau}(s_{\tau+1}, a'_{\tau+1})$, where \(\theta'\) denotes a fixed (detached) copy of the parameters.
        
        \item the loss between predicted Q and target:
    \begin{align*}
        \mathcal{L}(\theta_\tau) &= \big( \mathcal{N}_{\theta_\tau}(s_\tau, a_\tau) - (r_\tau + \gamma \mathcal{N}_{\theta'_\tau}(s_{\tau+1}, a'_{\tau+1})\big)^2.
    \end{align*}
    \end{itemize}
    \STATE Backpropagate and update $\theta_\tau$: \begin{itemize}
        \item Compute gradient of loss w.r.t. network parameters: $\nabla_{\theta_\tau} \mathcal{L}(\theta_\tau)$.
        \item Update online network parameters using Adam optimizer:
    \begin{align*}
        \theta_{\tau+1} = \theta_{\tau} - \eta_\tau \nabla_{\theta_\tau} \mathcal{L}(\theta_\tau).
    \end{align*}
    \end{itemize}
    \STATE Update decay learning rate: $\eta_{\tau+1} = \eta_{\tau} \cdot \lambda$.
\ENDFOR
\STATE Set \(\widetilde{Q}_{\pi} = \mathcal{N}_{\theta_H}\) 
\STATE \textbf{Output:} Q-values estimate \(\widetilde{Q}_{\pi}\) for the policy-update step.
\end{algorithmic}
\end{algorithm}

\paragraph{Algorithm~\ref{alg:doubleDQN}: Double Deep Q-Network (Double DQN)}  
Standard Q-learning methods tend to overestimate action values due to maximization bias. Double DQN addresses this by decoupling action selection (from the online network) and evaluation (via a target network). This algorithm adds experience replay to stabilize training and periodically updates the target network to ensure learning consistency. It is particularly effective in high-dimensional environments where policy evaluation requires generalization.

\begin{algorithm}[H]
\caption{Estimation of Q-values via Double DQN}
\begin{algorithmic}[1] \label{alg:doubleDQN}
\STATE \textbf{Input:} state space $\mathcal{S}$, action space $\mathcal{A}$, transition kernel $\mathcal{T}: \mathcal{S}\times \mathcal{A} \mapsto \mathcal{S}$, reward function $\mathcal{R}: \mathcal{S}\times \mathcal{A} \mapsto [0, 1]$, stationary policy $\pi$, loss function $\mathcal{L}$, number of iterations $H$, learning rate decay factor and discount factor $\lambda, \gamma \in (0, 1)$.
\STATE \textbf{Initialize:} state $s_0$, learning rate $\eta_0$, replay buffer $\mathcal{B}$, target network parameters $\theta'_0 = \theta_0$, and online network $\mathcal{N}_{\theta_0}$ consisting of two fully connected layers with ReLU activations, followed by a linear output layer.
\FOR{$\tau = 0 \dots H-1$}
    \STATE Observe $s_\tau$, select and execute action $a_\tau \sim \pi(\cdot | s_\tau)$, receive reward $r_\tau$, and observe next state $s_{\tau+1}$.
    \STATE Store transition $(s_\tau, a_\tau, r_\tau, s_{\tau+1})$ in buffer $\mathcal{B}$.
    \STATE Sample a minibatch $\{(s_j, a_j, r_j, s_{j+1})\}$ from $\mathcal{B}$.
    \STATE Compute
    \begin{itemize}
    \item Predicted Q-value: $\mathcal{N}_{\theta_\tau}(s_j, a_j)$.
        \item Target action: $a'_{\tau+1}\sim \pi(\cdot | s_{j+1})$.
        \item Target Q-value: $\mathcal{N}_{\theta'_\tau}(s_{\tau+1}, a'_{\tau+1})$.
        
        \item Loss:
        \begin{align*}
            \mathcal{L}(\theta_\tau) &= \frac{1}{|\mathcal{B}|} \sum_{j} \big( \mathcal{N}_{\theta_\tau}(s_j, a_j) - r_j + \gamma \mathcal{N}_{\theta'_\tau}(s_{j+1}, a'_{\tau+1}) \big)^2.
        \end{align*}
    \end{itemize}
    \STATE Backpropagate and update $\theta_\tau$: \begin{itemize}
        \item Compute gradient of loss w.r.t. network parameters: $\nabla_{\theta_\tau} \mathcal{L}(\theta_\tau)$.
        \item Update online network parameters using Adam optimizer:
    \begin{align*}
        \theta_{\tau+1} = \theta_{\tau} - \eta_\tau \nabla_{\theta_\tau} \mathcal{L}(\theta_\tau).
    \end{align*}
    \end{itemize}

    \STATE Periodically update target network parameters: $\theta'_\tau \leftarrow \theta_\tau$.
    \STATE Decay learning rate: $\eta_{\tau+1} = \eta_{\tau} \cdot \lambda$.
\ENDFOR
\STATE Set $\widetilde{Q}_{\pi} = \mathcal{N}_{\theta_H}$.
\STATE \textbf{Output:} Q-values estimate $\widetilde{Q}_{\pi}$ for the policy-update step.
\end{algorithmic}
\end{algorithm}

\subsection{Standard Algorithms} \label{sec:alg:std}

\paragraph{Algorithm~\ref{alg:QL}: Q-Learning}  
Q-learning is a foundational off-policy algorithm for estimating optimal action values. It operates via value iteration, using greedy action selection with $\epsilon$-greedy exploration. Despite its simplicity, it forms the conceptual basis for many modern algorithms like DQN. We include it here for completeness and as a reference baseline for tabular settings.

    \begin{algorithm}[H]
\caption{Q-Learning}
\begin{algorithmic}[1]\label{alg:QL}
\STATE \textbf{Input:} state space $\mathcal{S}$, action space $\mathcal{A}$, trasition kernel $\mathcal{T}: \mathcal{S}\times \mathcal{A} \rightarrow \mathcal{P}(\mathcal{S})$, reward function $\mathcal{R}: \mathcal{S}\times \mathcal{A} \rightarrow [0, 1]$, stationary policy $\pi$, number of iterations $H$, learning rate, and discount factor \(\alpha, \gamma \in (0, 1)\), starting exploration rate and exploration rate decay factor \(\epsilon_0, \lambda \in (0, 1)\).
\STATE \textbf{Initialize:} state $s_0$, \({Q}_{\pi, 0}(s, a) = 0\) for all \(s \in \mathcal{S}, \, a \in \mathcal{A}\)
\FOR{$\tau = 0 \dots H-1$}
    \STATE Observe \( s_\tau \), select and execute action \( a_\tau \) such that  
\[
a_\tau =
\begin{cases} 
\text{random action}, & \text{with probability } \epsilon_\tau, \\
\arg\max_a Q_{\pi, \tau}(s_\tau, a), & \text{otherwise}.
\end{cases}
\]
Receive reward \( r_\tau \), and observe next state \( s_{\tau+1} \).

    \STATE Update the estimate:
    \begin{align*}
        {Q}_{\pi, \tau+1}(s_\tau, a_\tau) &= (1-\alpha){Q}_{\pi, \tau}(s_\tau, a_\tau) + \alpha \left( r_\tau + \gamma \max_a{Q}_{\pi, \tau}(s_{\tau+1}, a) \right)
        \\
        {Q}_{\pi, \tau+1}(s, a) &= {Q}_{\pi, \tau}(s, a) \;\;\;\;\;\;\;\;\;\;\;\; \forall s, a \ne s_\tau, a_\tau
    \end{align*}

    \STATE Set $\epsilon_{\tau+1} = \lambda \cdot \epsilon_\tau$

\ENDFOR
\STATE \textbf{Output:} \({Q}_{\pi} = {Q}_{\pi, H}\).
\end{algorithmic}
\end{algorithm}

\section{Formal description of the MDP and of the expert policies}
\label{appendix:matching}

This appendix provides the formal details of the stochastic matching model summarized in Section~5.

\paragraph{Action space}
The actions consist of making a match, i.e., selecting two indexes in $[I]$,
putting the item in its queue, which we denote by $\vDash$,
or trashing it, which we denote by $\boxtimes$ if its queue is already full.
That is,
\[
\gA = [I]\times [I] \cup \{ \vDash,\,\boxtimes\}\,
\]
More precisely,
the action taken $a_t$ lies in $[I]\times [I]$
if a match can be made between $a_t(1)$ and an item of class $a_t(2)$:
this requires compatibility between $a_t(1)$ and $a_t(2)$ (as indicated by the compatibility
graph), and the availability of a least one item in the queue of both nodes, i.e., $\varrho_{t,a_t(1)} \geq 1$ and $\varrho_{t,a_t(2)} \geq 1$.
Otherwise, $a_t = \,\vDash$ if $\varrho_{i_t,t} \leq L-1$ and $a_t = \boxtimes$
if $\varrho_{i_t,t} = L$.

\paragraph{State Space}
The situation at the beginning of the round $ t \geq 0 $ is summarized by the triplet $ s_t = (\varrho_t, i_t, e_t) $, where:
\begin{itemize}
    \item $ \varrho_t = (\varrho_{t,i})_{i \in [I]} $ is the vector of all queue sizes;
    \item $e_t$ is the event occurring at time $t$, where $e_t\in E=\{\text{arrival, departure, relocation, None}\}$
    \item $i_t$ is the class where the event occurs;
\end{itemize}
The state space therefore lies in:
\[
\gS = [L]^I \times [I] \times E.
\]

\paragraph{Event Probabilities}  
For each queue $i$, three possible events can occur:  
\begin{itemize}  
    \item \textbf{Arrival:} Occurs at rate $\lambda_i$.  
    \item \textbf{Departure:} Occurs at rate $\mu_i$, provided there are customers in the queue.  
    \item \textbf{Relocation:} Occurs at rate $\nu_i$, provided there are customers in the queue.  
\end{itemize}  

To model item relocation within the system, we define $N_i$ as the ``next node'' for node $i$. Specifically, if an item does not depart the system from node $i$, it moves to node $N_i$. This mapping ensures that every node $i$ has a designated transfer destination when an item remains in the system.

\paragraph{Uniformization Rate}  
To ensure a well-defined discrete-time process, we introduce a finite uniformization rate, denoted by $\Lambda$. This rate is chosen as the sum of all arrival and movement rates:  
\[
\Lambda = \sum_{i \in [I]} (\lambda_i + \mu_i L + \nu_i L).
\]
By selecting $\Lambda$ in this manner, we account for the possibility that no event occurs at certain time steps.  

\paragraph{Event Probabilities}  
At each time step $t$, the probability of an event $e_t$ occurring is given by:  
\begin{align}\label{eq:dym_prob}
    \mathbb{P}(e_t = X) = \begin{cases} 
    \lambda_i/\Lambda, & \text{if } X \text{ is an arrival at queue } i, \\[6pt]
    \mu_i \varrho_{t,i}/\Lambda, & \text{if } X \text{ is a departure from queue } i, \\[6pt]
    \nu_i \varrho_{t,i}/\Lambda, & \text{if } X \text{ is a relocation from queue } i, \\[6pt]
    1 - \sum_{i=1}^{n} (\lambda_i + \mu_i \varrho_{t,i} + \nu_i \varrho_{t,i})/\Lambda, & \text{if } X \text{ is no event}.
    \end{cases}
\end{align}
 
 An \textit{arrival} introduces a new item into the system at queue $i$. A \textit{departure} removes an item from queue $i$, provided the queue is not empty, while a \textit{relocation} moves an item from queue $i$ to its designated next node $N_i$, unless it exits the system.There are $\varrho_{t,i}$ items in queue $i$, and each item may independently depart, either by leaving the queue through a \textit{departure} with rate $\mu_i$ or by being relocated with rate $\nu_i$. If no event occurs, the system remains unchanged for that time step.

This formulation ensures that every possible event, including the case where no event occurs, is accounted for within the uniformized framework.

\paragraph{Transition Kernel}
The transition to the next state is governed by the transition kernel $ \gT : \gS \times \gA \to \gP(\gS) $. Given a state $ s_t = (\varrho_t, i_t, e_t) $ and a possible action $ a_t $, the subsequent state $ s_{t+1} = (\varrho_{t+1}, i_{t+1}, e_{t+1}) $ is generated as follows:

\begin{itemize}
    \item The class $i_{t+1}$ is drawn conditionally on $e_t$, with the probability:
    \[
    \mathbb{P}(\text{next event occurs at } i_{t+1} \mid e_t) = \frac{\lambda_{i_{t+1}} + \mu_{i_{t+1}} \varrho_{i_{t+1}} + \nu_{i_{t+1}} \varrho_{i_{t+1}}}{\Lambda}.
    \]

    \item The queue sizes $ \varrho_t $ are updated to $ \varrho_{t+1} $ as follows:
    \begin{align}\label{eq:transition_matrix}
    \varrho_{t+1} = \varrho_t 
    &+ 
    \begin{cases}
        -\mathbf{1}_{i_t} & \text{if } e_t = \text{departure at queue } i_t,\\[6pt]
        \mathbf{1}_{i_t} & \text{if } e_t = \text{arrival at } i_t, \\[6pt]
        -\mathbf{1}_{i_t} + \mathbf{1}_{N_{i_t}} & \text{if } e_t = \text{relocation at } i_t
    \end{cases} 
   \notag \\\\
    &+ 
    \begin{cases}
        -\mathbf{1}_{a_t(1)} - \mathbf{1}_{a_t(2)} & \text{if } a_t \in [I] \times [I],\\[6pt]
        -\mathbf{1}_{i_t} & \text{if } e_t = \text{arrival at } i_t \text{ and } a_t = \boxtimes, \\[6pt]
        -\mathbf{1}_{N_{i_t}} & \text{if } e_t = \text{relocation at } i_t \text{ and } a_t = \boxtimes
    \end{cases}
    \notag
\end{align}
\medskip
    Here, $ \mathbf{1}_{k} $ is the indicator vector for $ k \in [I] $, which is zero everywhere except at component $ k $, where it equals 1.

    \item The next event $ e_{t+1} $ is then sampled with the probabilities defined in Equation \ref{eq:dym_prob}.
\end{itemize}

This process fully determines the transition kernel $ \gT $.

\paragraph{Reward function}
We now describe the deterministic reward (cost) function $r : \gS \times \gA \to [-LI, M]$, where $M$ denotes its upper range. Positive rewards are obtained when a match is made, but some matches may yield higher rewards than others. Costs for maintaining the queues will be incurred in all cases. The actions of placing an item in a queue or trashing it lead to the same values of the reward function. More precisely, for a given state $s = (\varrho, i, e)$ and an action $a$, the reward function is given by:

\[
r(s,a) = - D_{i} \cdot 1_{e=\text{departure}} - R_{i} \cdot 1_{e=\text{relocation}} + g_{a}
\]

where $1$ is the indicator function (equal to 1 when $e_t = \text{departure}$ and 0 otherwise).

Reward Components:

\begin{itemize}
    \item \textbf{$D_i$}: This term represents the cost associated with an item departing the system from queue $i$. A departure event negatively impacts the reward, as it may indicate the loss of an item or a missed opportunity for matching.
    
    \item \textbf{$R_i$}: This term accounts for the cost associated with an item being relocated within the system from queue $i$. Like departures, relocations are penalized because they involve moving an item out of its current queue, potentially delaying or complicating future matching opportunities.

    \item \textbf{$g_{a}$}: This term captures the reward associated with taking a specific action $a$. For example, this could be the reward for successfully matching items from different queues or placing an item in a queue for future matching. The exact form of $g_{a}$ is context-dependent and is designed to encourage desired behaviors, such as making high-value matches.
\end{itemize}

In this setting, the reward structure balances the costs of departures and relocations with the rewards for successful actions like matching items. The reward for matching is indirectly captured through the term $g_{i,a}$, while departures and relocations are explicitly penalized. The goal is to incentivize the learner to make decisions that maximize matching efficiency, while minimizing the negative impacts of item departures and relocations.

Thus, the learner’s task is to navigate this cost-reward trade-off over time by learning an optimal policy. This policy should manage the complexities of the system, including efficiently handling the queues, making successful matches, and minimizing the costs of departures and relocations.

\paragraph{Distribution on initial state}
The initial state $s_0 = (\varrho_0,\,i_0, e_0)$
consists of an index $i_0$ drawn at random according to $\lambda$
and of an empty queue: $\varrho_{0,j} = 0$ for all $j \in [I]$.
We denote by $\mu_0$ the corresponding distribution of $s_0$.

\paragraph{Expert policies}
The first expert policy $\pi_1$ is called \emph{match the longest}:
if at least one match is possible, this policy always chooses the class with the most items in its queue
(ties broken based on the payoffs).
The other policies are of \emph{edge-priority} type and select matches according to some intrinsic priority order defined on the edges of the compatibility graph.
If at least one match is possible, the expert policy $\pi_2$ chooses the match leading to the largest payoff (ties broken based on queue lengths).
Finally, the expert policy $\pi_3$ also follows the greedy policy described by $\pi_2$, but only for the items beloging to a given set of classes $P(\pi_3)$. Otherwise, if no match is possible,
all expert policies described above add the item to its queue, if the maximal length $L$ of the latter is not achieved yet;
and in last resort, they trash the item.
We denote by $\gN(i) \subseteq [I]$ the item classes that are compatible with class $i$,
and formalize the expert policies $\pi_1,\,\pi_2,\,\pi_3: \gS \to \gP(\gA)$,
For a given state $s = (\varrho, i, e)$, we denote by
\[
\gM(s) = \bigl\{ j \in \gN(i) : \varrho_j \geq 1 \bigr\}
\]
the prospective matches. Neglecting the tie-breaking rules, $\mbox{if} \ \gM(s) \neq \emptyset$
\begin{align*}
& \pi_1(\,\cdot\,|s) = \diracP{\argmax_{j \in \gM(s)} \varrho_j},
\pi_2(\,\cdot\,|s) = \diracP{\argmax_{j \in \gM(s)} g_{i,j}}, \\
& \pi_3(\,\cdot\,|s) = \diracP{\argmax_{j \in P(\pi_3) \cap \gM(s)} g_{i,j}}\,,
\end{align*}
where ${\dirac}(k)$ denotes the Dirac mass at $j$;
otherwise, $\forall k \in \{1,2,3\}$,
\begin{align*}
   & \pi_k(\,\cdot\,|s) = {\dirac}(\vDash) \ \mbox{ if } \varrho_i \leq L-1 \\ &\pi_k(\,\cdot\,|s) = {\dirac}(\boxtimes) \ \mbox{ if } \varrho_i = L\,.
\end{align*}

We define $\gN(i) \subseteq [I]$ as the set of item classes that are compatible with class $i$, and we formalize the expert policies $\pi_1$, $\pi_2$, and $\pi_3$ as mappings from states to sets of actions, i.e., $\pi_1, \pi_2, \pi_3: \gS \to \gP(\gA)$.

For a given state $s = (\varrho, i, e)$, we define the prospective matches $\gM(s)$ as follows:
\[
\gM(s) = \bigl\{ j \in \gN(i) : \varrho_j \geq 1 \bigr\}
\]
Neglecting the tie-breaking rules, $\mbox{if} \ \gM(s) \neq \emptyset$
\begin{align*}
& \pi_1(\,\cdot\,|s) = \diracP{\argmax_{j \in \gM(s)} \varrho_j},
\pi_2(\,\cdot\,|s) = \diracP{\argmax_{j \in \gM(s)} g_{i,j}}, \\
& \pi_3(\,\cdot\,|s) = \diracP{\argmax_{j \in P(\pi_3) \cap \gM(s)} g_{i,j}}\,,  \pi_4(\,\cdot\,|s) = \frac{1}{|\gM(s)|} \sum_{j \in \gM(s)} \diracP{j}
,
\end{align*}
where ${\dirac}(k)$ denotes the Dirac mass at $j$;
otherwise, $\forall k \in \{1,2,3\}$,
\begin{align*}
   & \pi_k(\,\cdot\,|s) = {\dirac}(\vDash) \ \mbox{ if } \varrho_i \leq L-1 \\ &\pi_k(\,\cdot\,|s) = {\dirac}(\boxtimes) \ \mbox{ if } \varrho_i = L\,.
\end{align*}

\section{Simulations} \label{sec:sim}

In this appendix we provide detailed simulation setups, parameter values, and additional results. The aim is to make the experimental sections in the main text fully reproducible while guiding the reader through the logic of each testbed.

\subsection{Diamond Graph} \label{sec:sim:setting1}

The diamond graph (Figure~\ref{fig:diamond_graph}) is a stylized environment with four nodes and five edges. Despite its small size, it captures the key challenges of matching: balancing immediate rewards against long-term opportunities. This controlled setting allows us to test convergence properties and the effect of different learning strategies.

\paragraph{Simulation Settings} Each node represents a queue with stochastic arrivals, and each edge represents a feasible match with an associated reward. The specific rates are summarized in Table~\ref{tab:diamond_params}, while global parameters are listed in Table~\ref{tab:global_params_diamond}.

\begin{table}[t]
\centering
\caption{Parameters for the diamond network (Figure~\ref{fig:diamond_graph}).}
\label{tab:diamond_params}
\begin{tabular}{c c c}
\toprule
\textbf{Node} & \textbf{Arrival Rate $\lambda_i$} & \textbf{Other Rates} \\
\midrule
1 & 0.125 & $\mu_1 = 0$, $\nu_1 = 0$ \\
2 & 0.225 & $\mu_2 = 0$, $\nu_2 = 0$ \\
3 & 0.150 & $\mu_3 = 0$, $\nu_3 = 0$ \\
4 & 0.050 & $\mu_4 = 0$, $\nu_4 = 0$ \\
\bottomrule
\end{tabular}

\vspace{0.5em}
\begin{tabular}{c c}
\toprule
\textbf{Edge} $(i,j)$ & \textbf{Reward $g_{i,j}$}  \\
\midrule
(1,2) & 10 \\
(2,4) & 200 \\
(2,3) & 50 \\
(1,3) & 1 \\
(3,4) & 20 \\
\bottomrule
\end{tabular}
\end{table}

\begin{table}[t]
\centering
\caption{Global parameters for the diamond graph.}
\label{tab:global_params_diamond}
\begin{tabular}{l c}
\toprule
\textbf{Parameter} & \textbf{Value} \\
\midrule
Queue capacity $L$ & 5 \\
Discount factor $\gamma$ & 0.8 \\
\bottomrule
\end{tabular}
\end{table}

\paragraph{Parameter Study} We first examine sensitivity to learning-rate schemes. Three orchestration strategies are tested: polynomial potential, exponential potential (fixed rate), and exponential potential with varying rate. Figure~\ref{fig:params} shows how each scheme influences the evolution of values across hyperparameters.

\begin{figure}[H]
    \centering
    \begin{subfigure}[h]{0.45\textwidth}
        \centering
        \includegraphics[width=\textwidth]{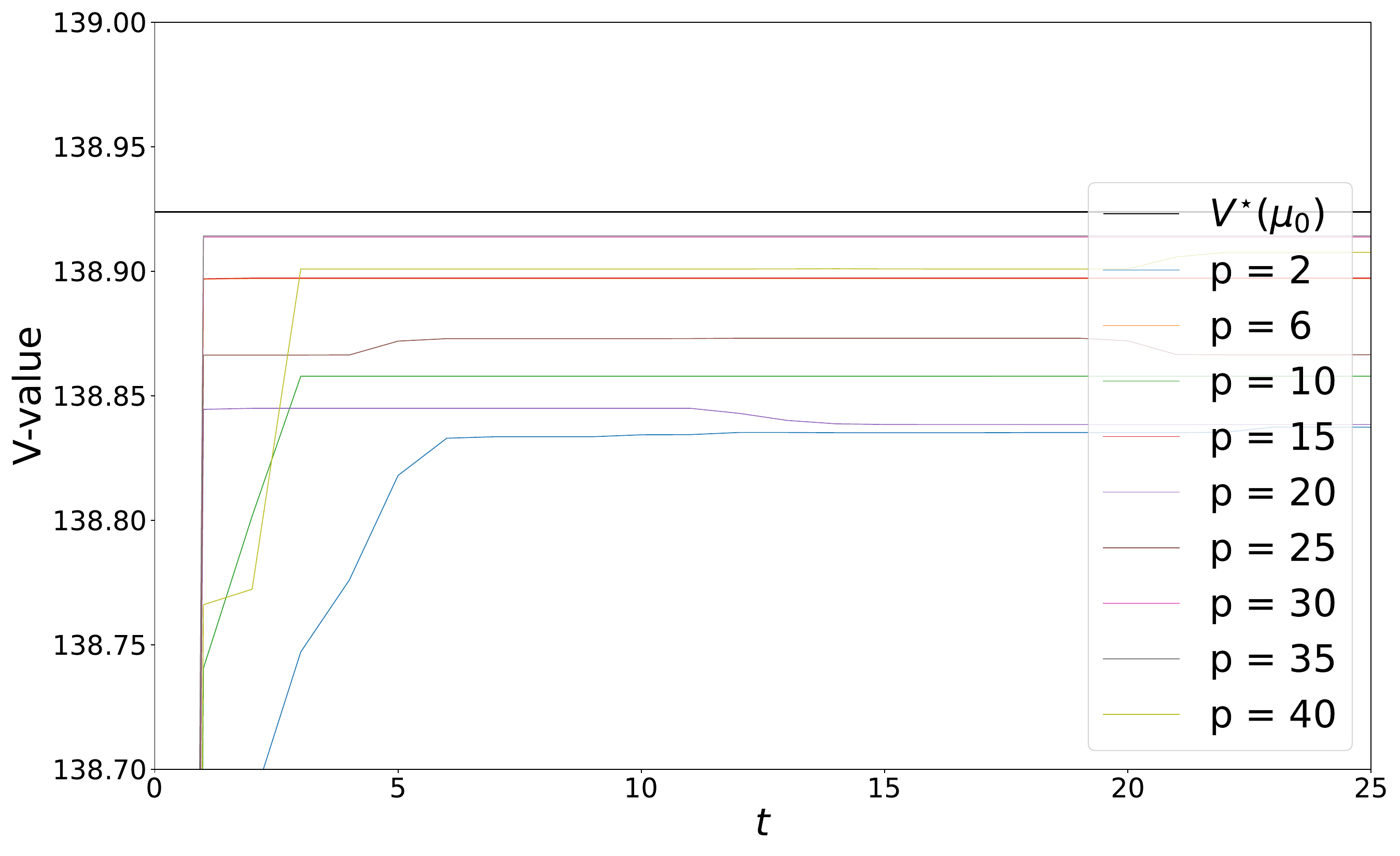}
        \caption{Polynomial Potential.}
    \end{subfigure}
    \begin{subfigure}[h]{0.45\textwidth}
        \centering
        \includegraphics[width=\textwidth]{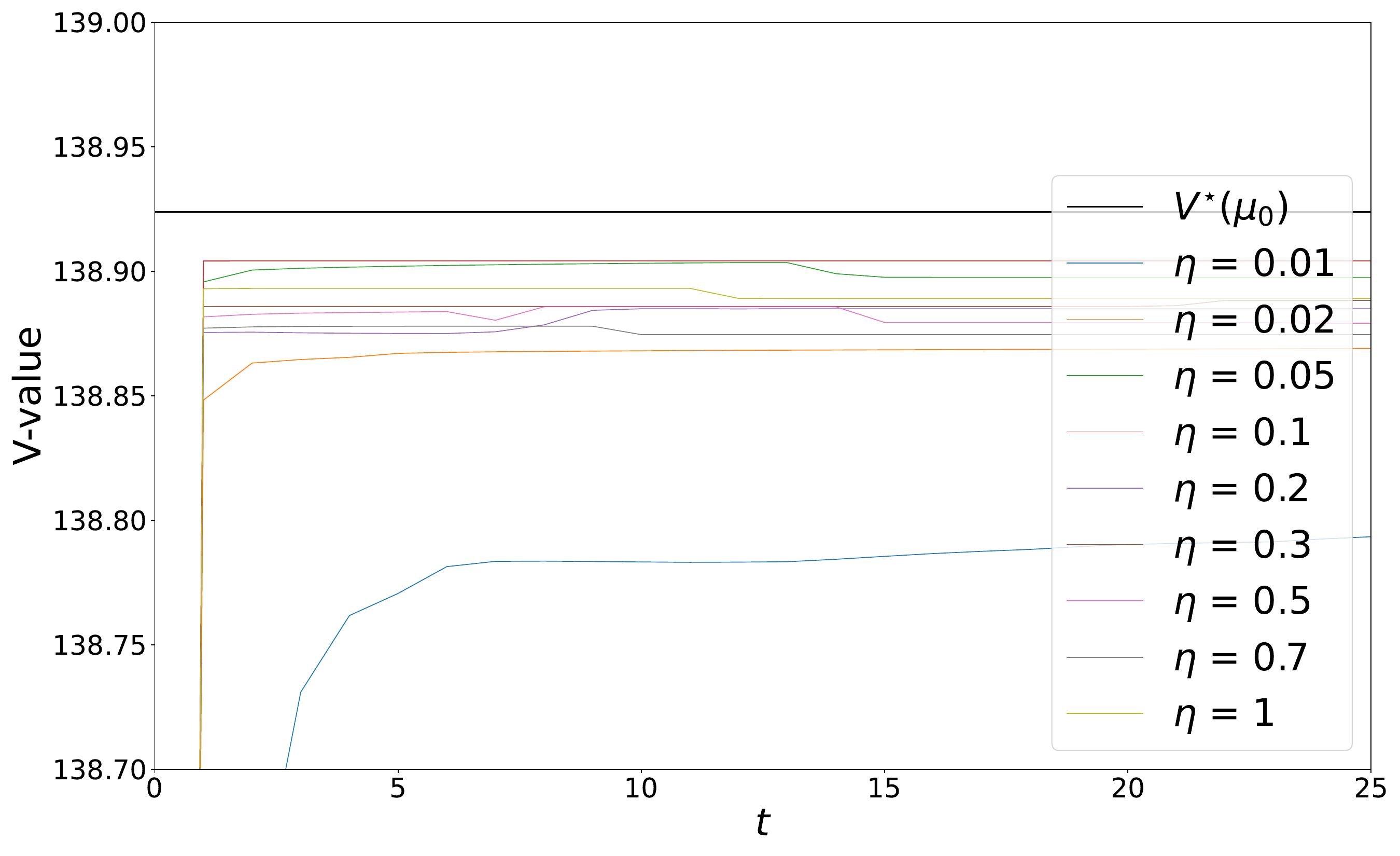}
        \caption{Exponential Potential ($\eta$ fixed).}
    \end{subfigure}
    \begin{subfigure}[h]{0.45\textwidth}
        \centering
        \includegraphics[width=\textwidth]{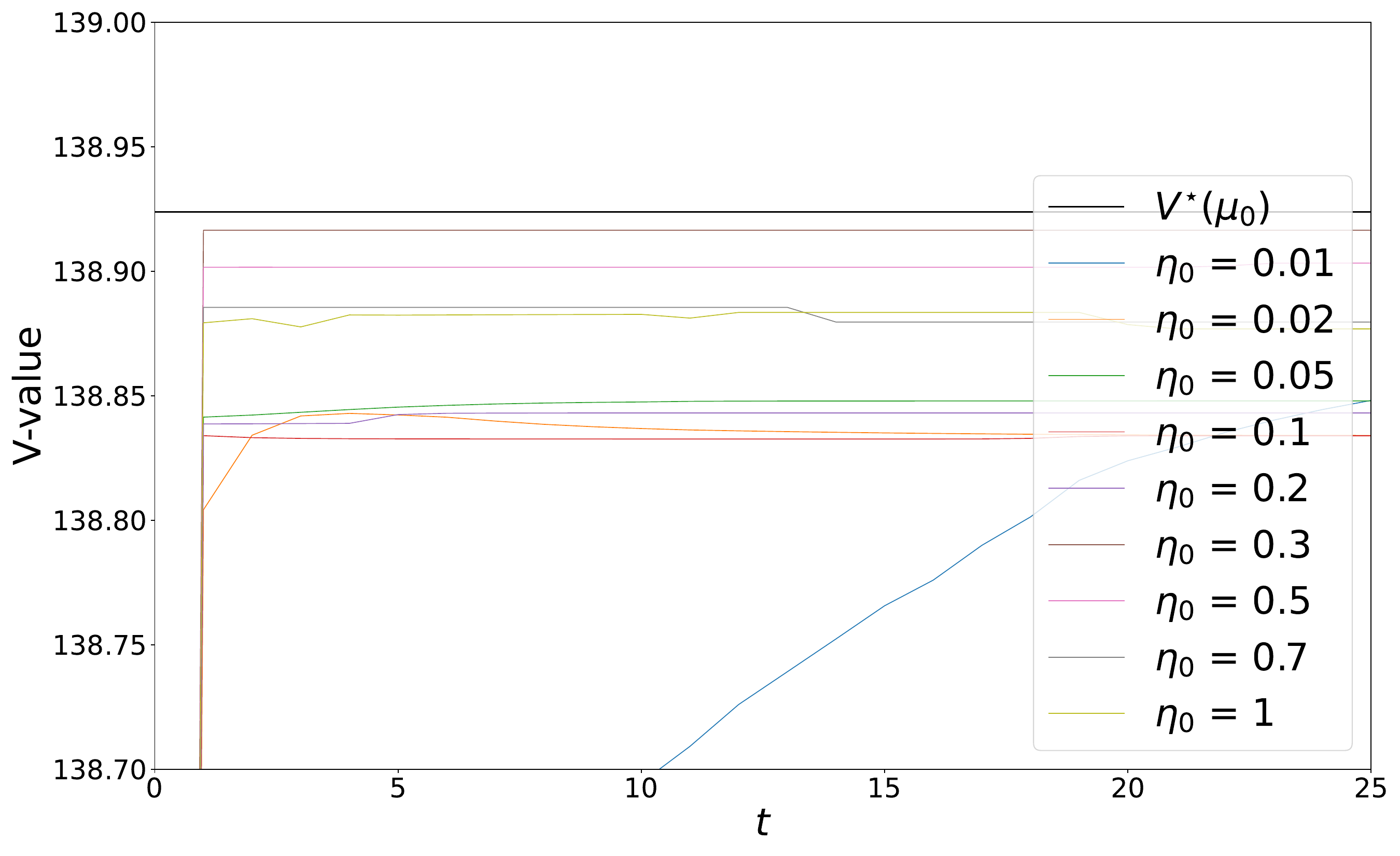}
        \caption{Exponential Potential ($\eta_t$ varying).}
    \end{subfigure}
    \caption{Evolution of values under different orchestration strategies and hyperparameter settings.}
    \label{fig:params}
\end{figure}

The exact parameter values are summarized in Table~\ref{tab:orchestration_params}. For comparison, we also include Q-learning (QL) with its own parameterization (Table~\ref{tab:ql_params}), and Figure~\ref{fig:params_ql} shows its performance across learning and exploration rates.

\begin{table}[t]
    \centering
    \caption{Learning-rate schemes for orchestration strategies.}
    \label{tab:orchestration_params}
    \begin{tabular}{l c}
        \toprule
        \textbf{Strategy} & \textbf{Learning Rate Scheme} \\
        \midrule
        Polynomial Potential & $p=30$ \\
        Exponential Potential & $\eta=0.1$ \\
        Exponential Potential (Varying Rate) & $\eta_0=0.3$ \\
        \bottomrule
    \end{tabular}
\end{table}

\begin{figure}[H]
    \centering
    \includegraphics[width=0.5\textwidth]{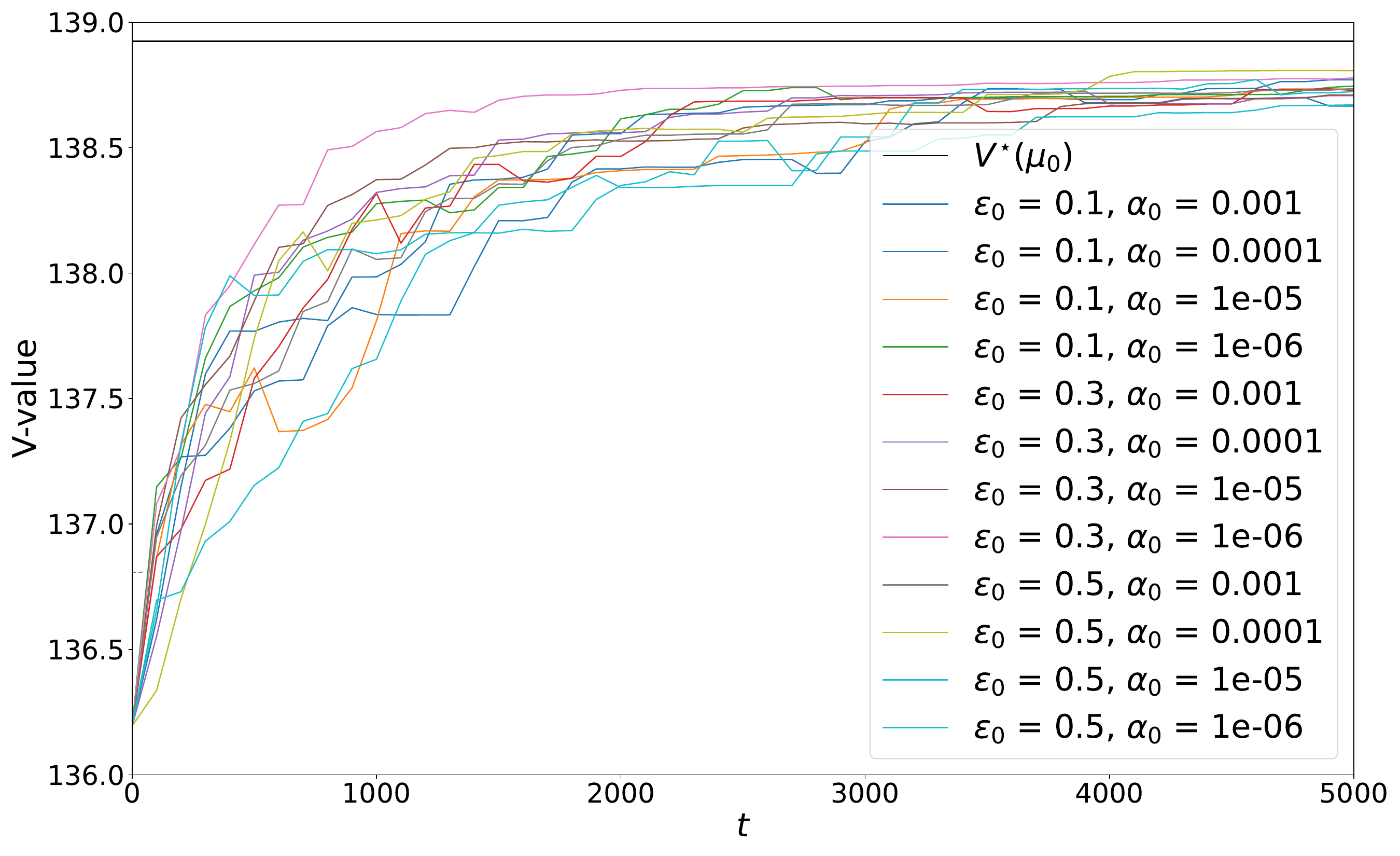}
    \caption{Values evolution of Q-learning across hyperparameter settings.}
    \label{fig:params_ql}
\end{figure}

\begin{table}[t]
    \centering
    \caption{Q-learning parameters.}
    \label{tab:ql_params}
    \begin{tabular}{l c}
        \toprule
        \textbf{QL Parameter} & \textbf{Value} \\
        \midrule
        Starting Learning Rate $\alpha_0$ & $1\cdot10^{-6}$ \\
        Starting Exploration Rate $\epsilon_0$ & $0.3$ \\
        Exploration Decay Factor $\lambda$ & $0.8$ \\
        \bottomrule
    \end{tabular}
\end{table}

\paragraph{Comparisons} We evaluate orchestration versus two baselines:
\begin{itemize}
    \item \textbf{Orchestration vs Q-learning on experts.} Figure~\ref{fig:enter-label} shows that orchestration strategies outperform QL even when both operate over the same set of expert policies.
    \item \textbf{Orchestration vs direct tabular learning.} When the number of direct actions is comparable to the number of experts, orchestration provides little additional benefit (Figure~\ref{fig:enter-label2}), highlighting that its main advantage arises in more complex environments.
\end{itemize}

\begin{figure}[H]
    \centering
    \includegraphics[width=0.55\linewidth]{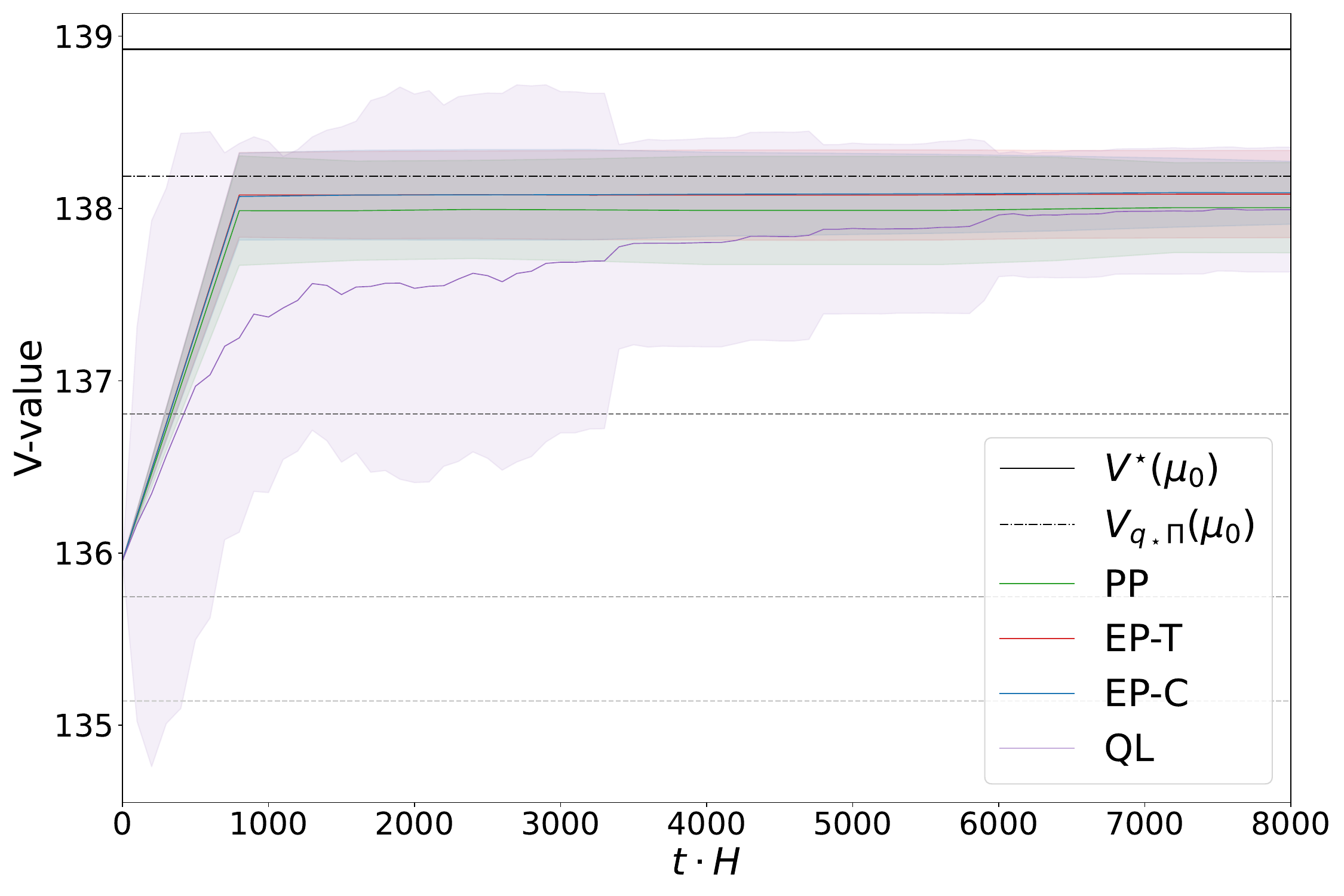}
    \caption{Comparison of orchestration of experts against QL policy.}
    \label{fig:enter-label}
\end{figure}

\begin{figure}[H]
    \centering
    \includegraphics[width=0.55\linewidth]{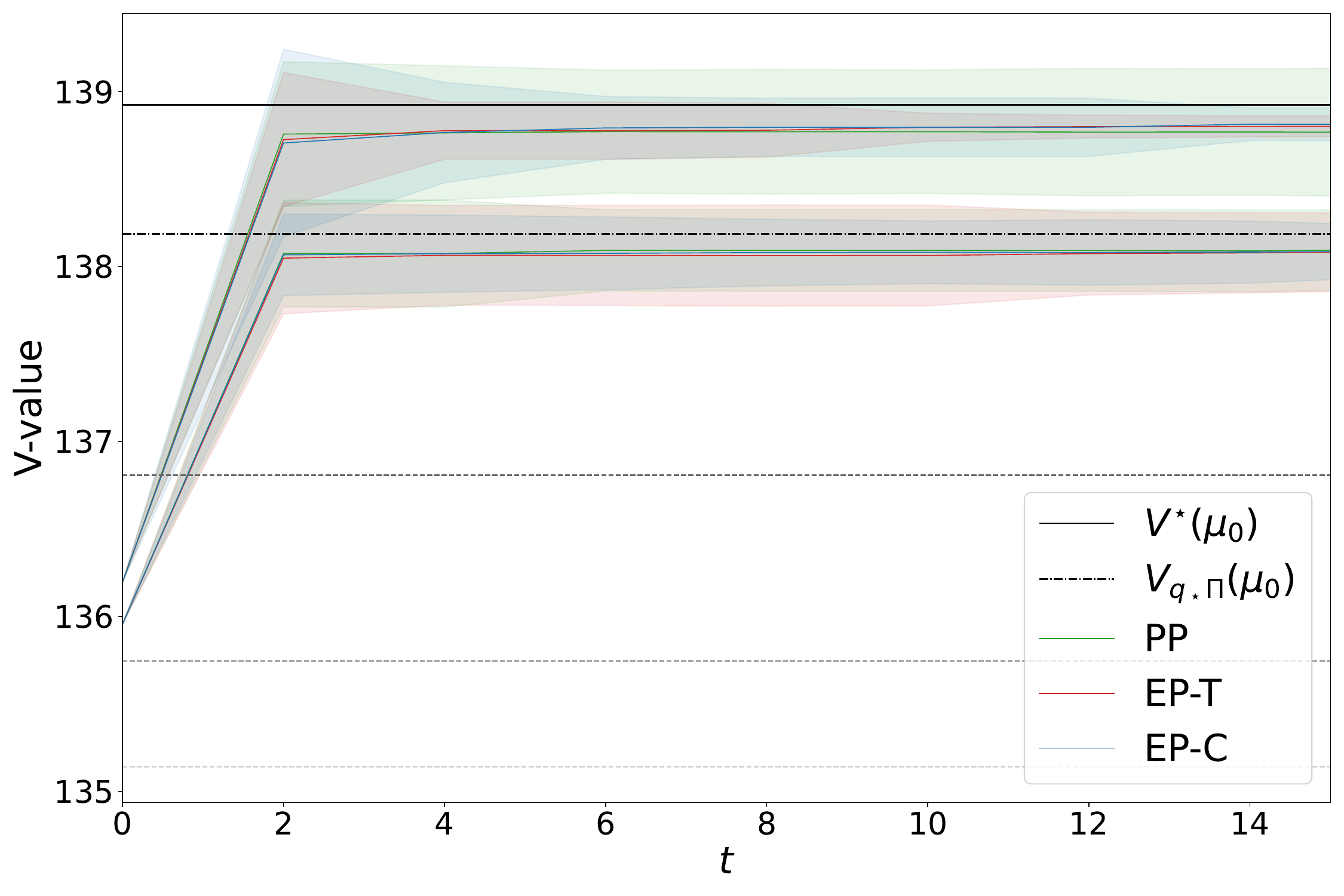}
    \caption{Comparison of orchestration of experts against direct tabular learning.}
    \label{fig:enter-label2}
\end{figure}

\subsection{Organ Exchange Model}\label{sec:sim:setting2}

We next turn to a realistic organ exchange model (Figure~\ref{fig:org_don_graph}). Here, patients and donors are grouped by blood type and urgency level. This environment is far richer than the diamond graph: arrivals are heterogeneous, urgency levels evolve over time (via relocation), and rewards are asymmetric, reflecting medical priorities.

\subsubsection{Identifying the Best Among Many}

\paragraph{Simulation Settings} Table~\ref{tab:node_params} lists node-specific parameters (arrival, departure, relocation), while Table~\ref{tab:urgency_params} gives urgency-dependent rewards and costs. Global settings are summarized in Table~\ref{tab:global_params}.

\begin{table}[t]
\centering
\caption{Node-specific parameters}
\label{tab:node_params}
\begin{tabular}{c c c c c}
\toprule
\textbf{Node} & \textbf{Urgency Level} & $\boldsymbol{\lambda_i}$ (Arrival) & $\boldsymbol{\mu_i}$ (Departure) & $\boldsymbol{\nu_i}$ (Relocation) \\
\midrule
1  & Donor    & 0.1 & 0 & - \\
2  & Donor & 0.002 & 0 & - \\
3  & Donor   & 0.082 & 0 & - \\
4  & Donor & 0.097 & 0 & - \\
5  & High    & 0.065 & 0.0008 & - \\
6  & Medium   & 0.029 & 0.0003 & 0.0005 \\
7  & Low    & 0.025 & 0.0001 & 0.0005 \\
8  & High & 0.098 & 0.0008 & - \\
9  & Medium   & 0.022  & 0.0003 & 0.0005 \\
10 & Low    & 0.011  & 0.0001 & 0.0005 \\
11 & High & 0.089 & 0.0008 & -\\
12 & Medium   & 0.124 & 0.0003 & 0.03 \\
13 & Low    & 0.0005 & 0.0001 & 0.0005 \\
14 & High & 0.067 & 0.0008 & - \\
15 & Medium   & 0.105 & 0.0003 & 0.0005 \\
16 & Low & 0.079 & 0.0001 & 0.0005 \\
\bottomrule
\end{tabular}
\end{table}

\begin{table}[t]
\centering
\caption{Urgency-level specific parameters}
\label{tab:urgency_params}
\begin{tabular}{l c c c}
\toprule
\textbf{Urgency Level $u$} & \textbf{Reward $g_u$} & \textbf{Departure Cost $D_u$} & \textbf{Relocation Cost $R_u$} \\
\midrule
Low    & 50 & 30 & 5 \\
Medium & 200 & 20 & 10 \\
High   & 1000 & 10 & - \\
\bottomrule
\end{tabular}
\end{table}

\begin{table}[t]
\centering
\caption{Global parameters}
\label{tab:global_params}
\begin{tabular}{l c}
\toprule
\textbf{Parameter} & \textbf{Value} \\
\midrule
Queue capacity $L$ & 5 \\
Discount factor $\gamma$ & 0.8 \\
Number of experts & 4 \\
\bottomrule
\end{tabular}
\end{table}

\paragraph{Discussion} In this setting, orchestration is particularly valuable: the action space is large, naive direct learning is slow, and individual experts are optimized for different sub-cases. Our results (main text, Section~\ref{sec:sim}) show that orchestration quickly identifies and matches the performance of the best expert.

\subsubsection{Improving Beyond the Best Expert}

We also consider a more structured organ exchange graph, incorporating blood type groups (O, A, B, AB). This increases the heterogeneity of arrivals and introduces further asymmetries.

\begin{table}[t]
\centering
\caption{Node-specific parameters}
\label{tab:node_params}
\begin{tabular}{c c c c c c}
\toprule
\textbf{Node} & \textbf{Group} & \textbf{Urgency Level} & $\boldsymbol{\lambda_i}$ (Arrival) & $\boldsymbol{\mu_i}$ (Departure) & $\boldsymbol{\nu_i}$ (Relocation) \\
\midrule
0  & O   & Donor   & 0.049 & 0 & - \\
1  & A   & Donor   & 0.018 & 0 & - \\
2  & B   & Donor   & 0.018 & 0 & - \\
3  & AB  & Donor   & 0.063 & 0 & - \\
4  & O   & High    & 0.049 & 0.008 & - \\
5  & O   & Medium  & 0.049 & 0.003 & 0.0005 \\
6  & O   & Low     & 0.049 & 0.001 & 0.005 \\
7  & A   & High    & 0.018 & 0.008 & - \\
8  & A   & Medium  & 0.018 & 0.003 & 0.0005 \\
9  & A   & Low     & 0.018 & 0.001 & 0.005 \\
10 & B   & High    & 0.018 & 0.008 & - \\
11 & B   & Medium  & 0.018 & 0.003 & 0.0005 \\
12 & B   & Low     & 0.018 & 0.001 & 0.005 \\
13 & AB  & High    & 0.063 & 0.008 & - \\
14 & AB  & Medium  & 0.063 & 0.003 & 0.0005 \\
15 & AB  & Low     & 0.063 & 0.001 & 0.005 \\
\bottomrule
\end{tabular}
\end{table}

\begin{table}[t]
\centering
\caption{Urgency-level specific parameters}
\label{tab:urgency_params}
\begin{tabular}{l c c c}
\toprule
\textbf{Urgency Level $u$} & \textbf{Reward $g_u$} & \textbf{Departure Cost $D_u$} & \textbf{Relocation Cost $R_u$} \\
\midrule
Donor  & 0    & 0  & 0 \\
Low    & 100  & 50 & 0 \\
Medium & 500  & 20 & 10 \\
High   & 1000 & 10 & 5 \\
\bottomrule
\end{tabular}
\end{table}

\begin{table}[t]
\centering
\caption{Global parameters}
\label{tab:global_params}
\begin{tabular}{l c}
\toprule
\textbf{Parameter} & \textbf{Value} \\
\midrule
Queue capacity $L$ & 15 \\
Discount factor $\gamma$ & 0.9 \\
Number of experts & 2 \\
\bottomrule
\end{tabular}
\end{table}

\paragraph{Discussion} Here, orchestration is not limited to “picking the best expert”: it learns combinations that improve upon both experts, exploiting complementary strengths. This is especially important in medical decision-making, where even small performance gains can translate into life-saving improvements.

\subsection{Computational Resources and Cost}

All experiments were conducted on a MacBook Pro (14‑inch, 2021) equipped with an Apple M1 Pro chip (10‑core CPU, 14‑core GPU) and 16\,GB of unified memory. The implementation is in Python using NumPy and PyTorch, and simulations were executed without GPU acceleration.

We evaluate two experimental settings: (1) a {Diamond Graph environment, representing a small-scale compatibility graph, and (2) a large-scale Organ Exchange Model} environment based on realistic transplantation networks.

In the Diamond Graph setting, we compare three algorithmic variants:
\begin{itemize}
    \item Tabular-Tabular: Tabular learning for both the policy and the advantage estimator,
    \item Tabular-NN:Tabular policy learning with a neural network for advantage estimation,
    \item NN-NN: Neural network approximation for both the policy and the advantage function.
\end{itemize}

Each configuration is run for $N = 200$ random seeds. For the policy-based methods, we perform 50 policy updates, and for each policy update, 15 estimation steps are conducted. The average training time per run is approximately 18.8 seconds (CPU time), with a wall time of 19.3 seconds.

For the Q-learning baseline, each run consists of 5 episodes of 150 steps each, using $N = 200$ random seeds. The average training time is approximately 11.1 seconds (CPU time), with a wall time of 11.3 seconds.

In the Organ Exchange Model, we employ only the NN-NN variant due to the large and structured state space. Each configuration is run for $N = 10$ random seeds. For the policy-based methods, we perform 200 policy updates, and for each policy update, 30 estimation steps are conducted. The average training time per run is approximately 4 hours.

All experiments were conducted locally on a single machine, and no cloud or distributed computing infrastructure was used. The computational cost remains moderate and reproducible on a high-end personal workstation.

\end{document}